\documentclass{article}
\usepackage[margin=1in]{geometry}
\usepackage[hidelinks]{hyperref}

\usepackage{amsmath, amssymb, amsthm, amsfonts}

\newtheorem{definition}{Definition}

\newtheorem{theorem}{Theorem}
\newtheorem*{theorem*}{Theorem}
\newtheorem*{prop*}{Proposition}

\DeclareMathOperator*{\argmin}{arg\,min}

\newtheorem{corollary}{Corollary}
\newtheorem{lemma}{Lemma}
\newtheorem*{lemma*}{Lemma}
\newtheorem*{PL-Inequality}{Polyak-Lojasiewicz (PL) Inequality}
\newtheorem*{corollary*}{Corollary}
\newcommand{\EE}{\mathbb{E}}

\usepackage{diagbox}
\usepackage{mathtools}
\usepackage{enumerate}
\usepackage{graphicx}
\usepackage{array}
\usepackage{tikz}
\usepackage{hyperref, doi}
\usepackage{epsf}
\usepackage{bbm}
\usepackage[utf8]{inputenc} 
\usepackage[T1]{fontenc}    
\usepackage{float}
\usepackage{booktabs}
\usepackage{caption}
\usepackage{subcaption}
\usepackage{array}
\usepackage{algpseudocode,algorithm,algorithmicx}
\usepackage{makecell}
\usepackage{selectp}

\newcolumntype{?}[1]{!{\vrule width #1}}
\usepackage{todonotes}

\makeatletter
\newcommand*\samethanks[1][\value{footnote}]{\footnotemark[#1]}
\newcommand{\specificthanks}[1]{\@fnsymbol{#1}}
\makeatother

\usepackage{titling}

\title{A Mechanism for Producing \\ Aligned Latent Spaces with Autoencoders}

\author{Saachi Jain\thanks{Equal Contribution} \textsuperscript{,}\thanks{Laboratory for Information \& Decision Systems, and Institute for Data, Systems, and Society, Massachusetts Institute of Technology} \and Adityanarayanan Radhakrishnan\samethanks[1] \textsuperscript{,}\samethanks[2] \and Caroline Uhler\samethanks[2]$^{~}$\\}

\date{\texttt{\{saachij, aradha, cuhler\}@mit.edu}}

\begin{document}

\maketitle

\begin{abstract}
Aligned latent spaces, where meaningful semantic shifts in the input space correspond to a translation in the embedding space, play an important role in the success of downstream tasks such as unsupervised clustering and data imputation.  In this work, we prove that linear and nonlinear autoencoders produce aligned latent spaces by stretching along the left singular vectors of the data.  We fully characterize the amount of stretching in linear autoencoders and provide an initialization scheme to arbitrarily stretch along the top directions using these networks.  We also quantify the amount of stretching in nonlinear autoencoders in a simplified setting.  We use our theoretical results to align drug signatures across cell types in gene expression space and semantic shifts in word embedding spaces.\footnote{Code available at \small{\url{https://github.com/uhlerlab/latent_space_alignment}}.}

\end{abstract}

\section{Introduction}
Modern deep networks are capable of learning rich latent spaces, i.e., vector spaces that capture similarity between data points (Ch.~15 of \cite{goodfellow2016deep}). These learned structured latent spaces play an important role in the success of several downstream inference tasks such as unsupervised clustering \cite{mukherjee2019clustergan}, topic modeling \cite{dieng2020topic}, and classification in low-data regimes \cite{rudolph2019structuring}.  


For many downstream learning tasks such as those in Figure \ref{fig: Overview Schematic}, a structured latent space in which meaningful semantic directions are \textit{aligned} is particularly useful.  For example, alignment of semantic directions (as in Figure \ref{fig: Overview Schematic}a) can be useful for identifying word analogies or performing word arithmetic with commonly used word embeddings such as GloVE \cite{pennington2014glove} and word2vec \cite{rong2014word2vec}.  

Recent work \cite{COVIDAutoencoding, TheisAutoencoding} empirically demonstrated that autoencoders are capable of automatically producing aligned spaces when embedding gene expression data, but the mechanism through which neural networks produce aligned latent spaces has been unclear.  

\begin{figure}
    \centering
    \begin{subfigure}{0.48\textwidth}
    \centering
    \includegraphics[width=.7\textwidth]{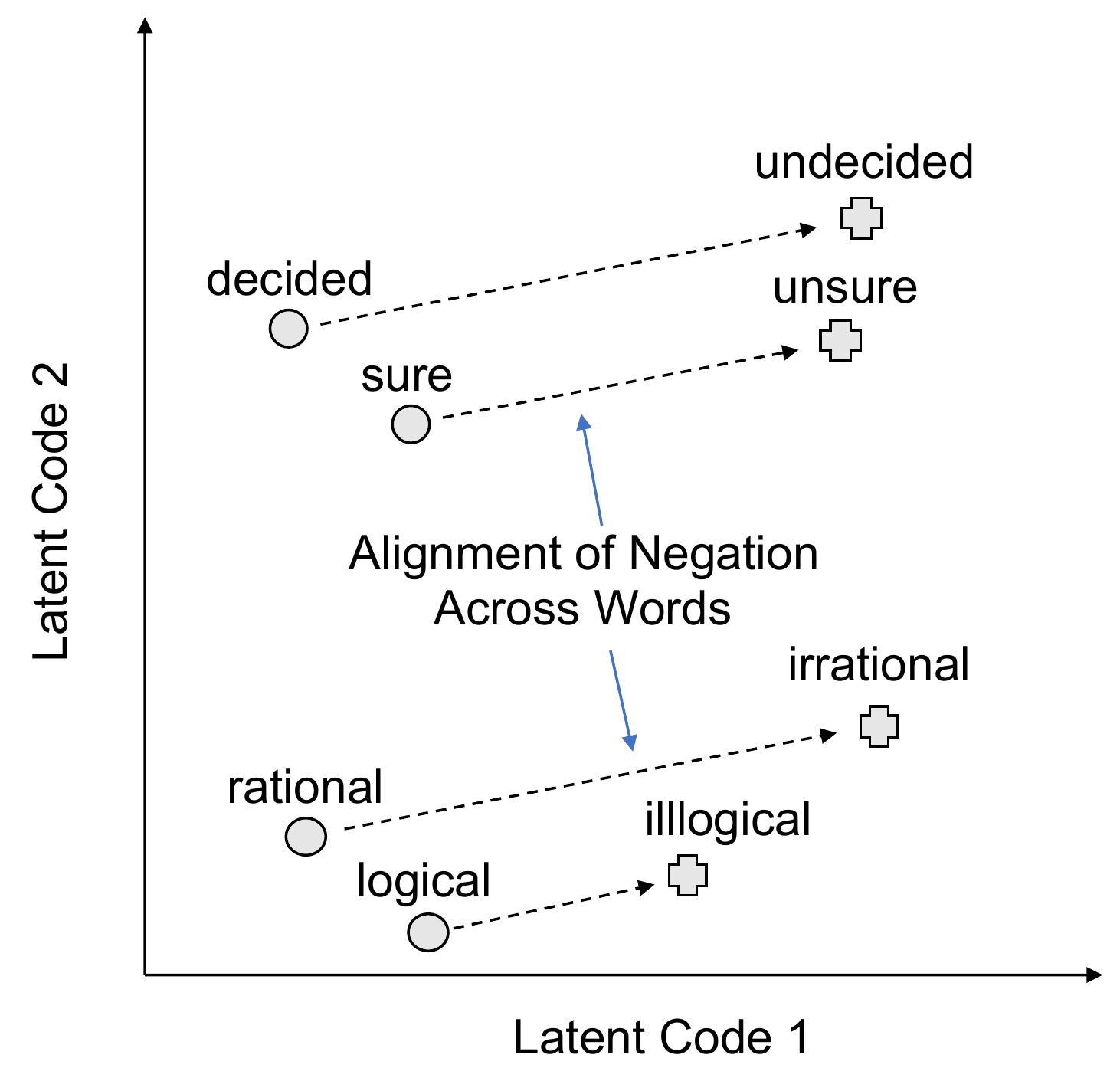}
    \caption{Word Embedding Space}
    \end{subfigure}
    \begin{subfigure}{0.48\textwidth}
    \centering
    \includegraphics[width=.935\textwidth]{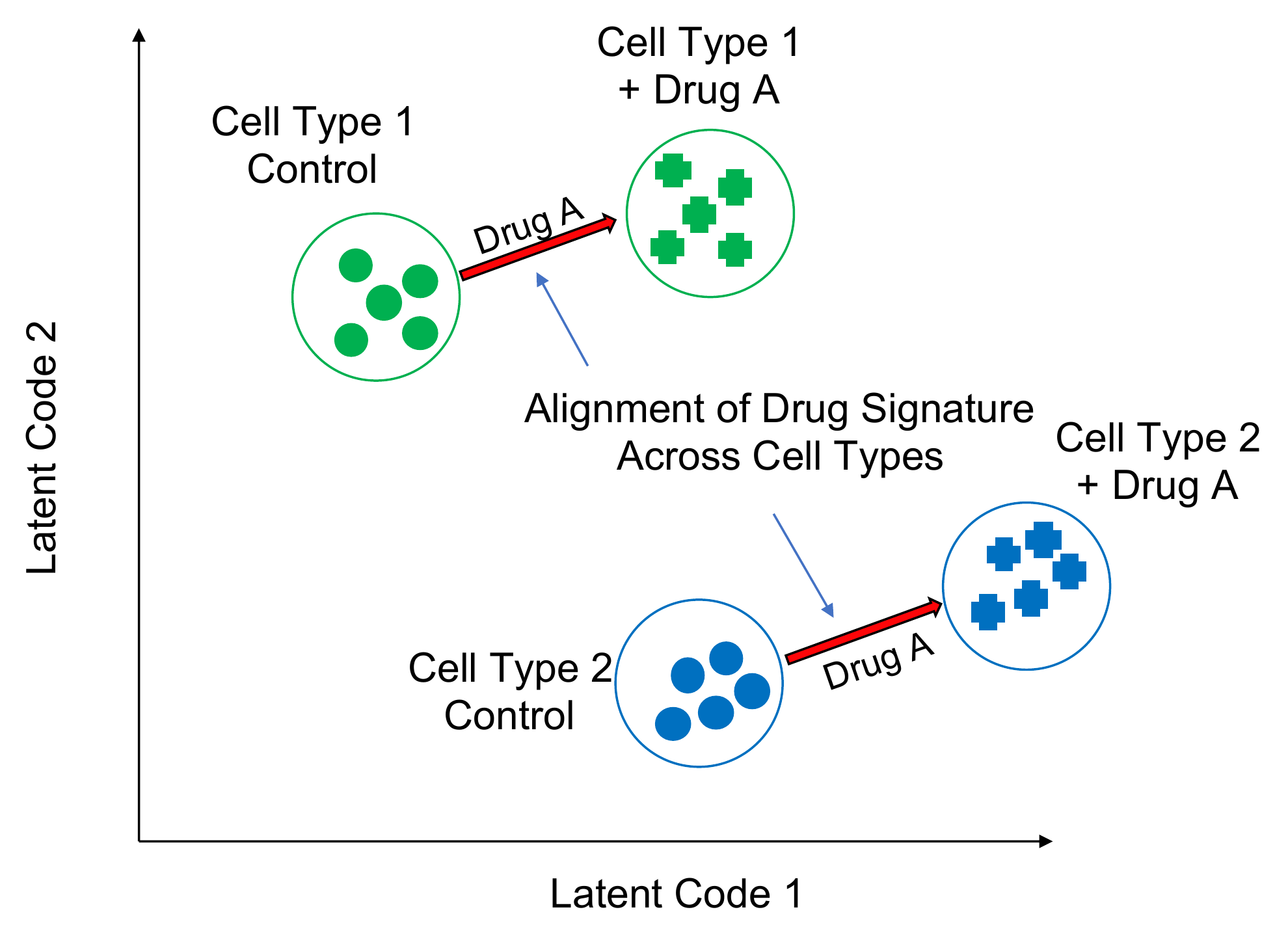}
    \caption{Gene Expression Embedding Space}
    \end{subfigure}
    \caption{Examples of aligned latent spaces. (a) Alignment of semantic directions (i.e. negation) in word embeddings. (b) Alignment of drug signatures across cell types in gene expression embeddings.}
    \label{fig: Overview Schematic}
\end{figure}





The main finding of this work is that autoencoders can produce aligned latent spaces by stretching along the top left singular vectors of the data.  The following are our contributions. 

\begin{enumerate}
    \item We prove that stretching along the top singular vectors of the data  leads to  alignment in the latent space.  We analyze the case of a mixture of Gaussians with unaligned major axes and prove that stretching along the top left singular vectors aligns these axes.  
    \item We prove that asymmetrically initialized linear autoencoders stretch along the top left singular vectors of the data, thereby producing aligned latent spaces.  We derive the stretching factor through a gradient flow analysis.  In a simplified setting, we prove that, under low rank initialization, non-linear autoencoders also stretch along the top left singular vectors of the data.  
    \item  We apply our theory to produce autoencoders on GloVe embeddings \cite{pennington2014glove} that strengthen alignment of meaningful semantic directions.  The resulting aligned latent spaces provide an improvement on word analogy tasks. We lastly apply our theory to produce autoencoders that align drug signatures across cell types on gene expression data from CMap~\cite{CMap}.  
\end{enumerate}

\section{Related Work}
\label{sec: Related Work}
Prior work \cite{TheisAutoencoding} demonstrated empirically that variational autoencoders \cite{VAE} can align the effect of perturbations across multiple cell types in single-cell gene expression data. In particular, the work demonstrated that under-parameterized, non-linear autoencoders automatically align the effect of stimulating cells by a single perturbation (IFN-$\beta$) in the latent space.  Follow-up work \cite{COVIDAutoencoding} demonstrated that over-parameterized 1-hidden layer autoencoders with leaky ReLU \cite{LeakyReLU} activation align the effect of multiple perturbations between cell types from the connectivity map (CMap) \cite{CMap}.  In this work, we prove that over-parameterized, 1-hidden layer autoencoders produce aligned latent spaces by stretching along the top left singular vectors of the data.  

Several recent works \cite{belkin2019reconciling, TwoModelsDoubleDescent, DeepDoubleDescent} study the behavior of over-parameterized neural networks in supervised learning.  While there are infinitely many interpolating solutions in the over-parameterized regime, these works demonstrated that an implicit bias through training encourages solutions that generalize.  Further works \cite{IdentityCrisis, RadhakrishnanOPAMemorization, RadhakrishnanOPAMemory} characterized the implicit bias of over-parameterized autoencoders in unsupervised settings.  In particular, \cite{RadhakrishnanOPAMemory, RadhakrishnanOPAMemorization} demonstrated that over-parameterized autoencoders learn solutions that are contractive around the training examples.  These works, however, did not analyze the latent space produced by over-parameterized autoencoders.  


It is well-known that linear autoencoders under a balanced spectral initialization \cite{Balancedness, Saxe2014, Saxe2019, ImplictRegularizationBach} learn the subspace spanned by the top principal components of the data \cite{AutoencodingPCA, AutoencodersPCAVectors}.  In this work, we analyze over-parameterized linear autoencoders under an asymmetric spectral initialization, i.e. one layer is initialized at 0 and the other is initialized using a matrix polynomial in the data.  Importantly, unlike prior spectral initialization analyses, our asymmetric initialization is easy to compute in practice.   In order to derive the solution at the end of training in linear autoencoders under spectral initialization, we use a gradient flow analysis, i.e. gradient descent with infinitesimal step size  \cite{ImplictRegularizationBach, RadhakrishnanOPAMemory, BalancednessJasonLee}.

\section{Producing Alignment through Stretching}
\begin{figure}
    \centering
    \begin{subfigure}{0.48\textwidth}
    \centering
    \includegraphics[width=.7\textwidth]{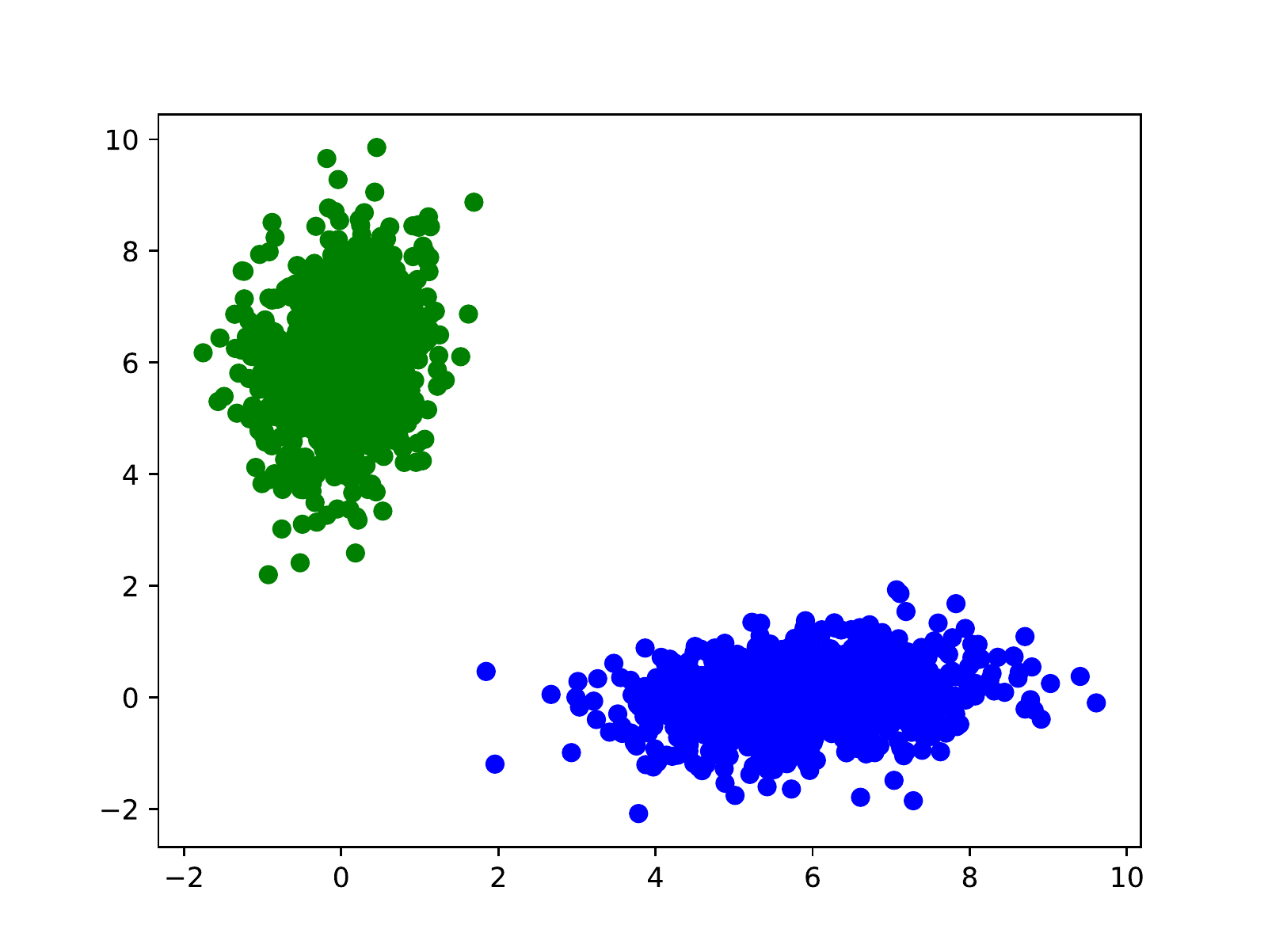}
    \caption{Original Space}
    \end{subfigure}
    \begin{subfigure}{0.48\textwidth}
    \centering
    \includegraphics[width=.7\textwidth]{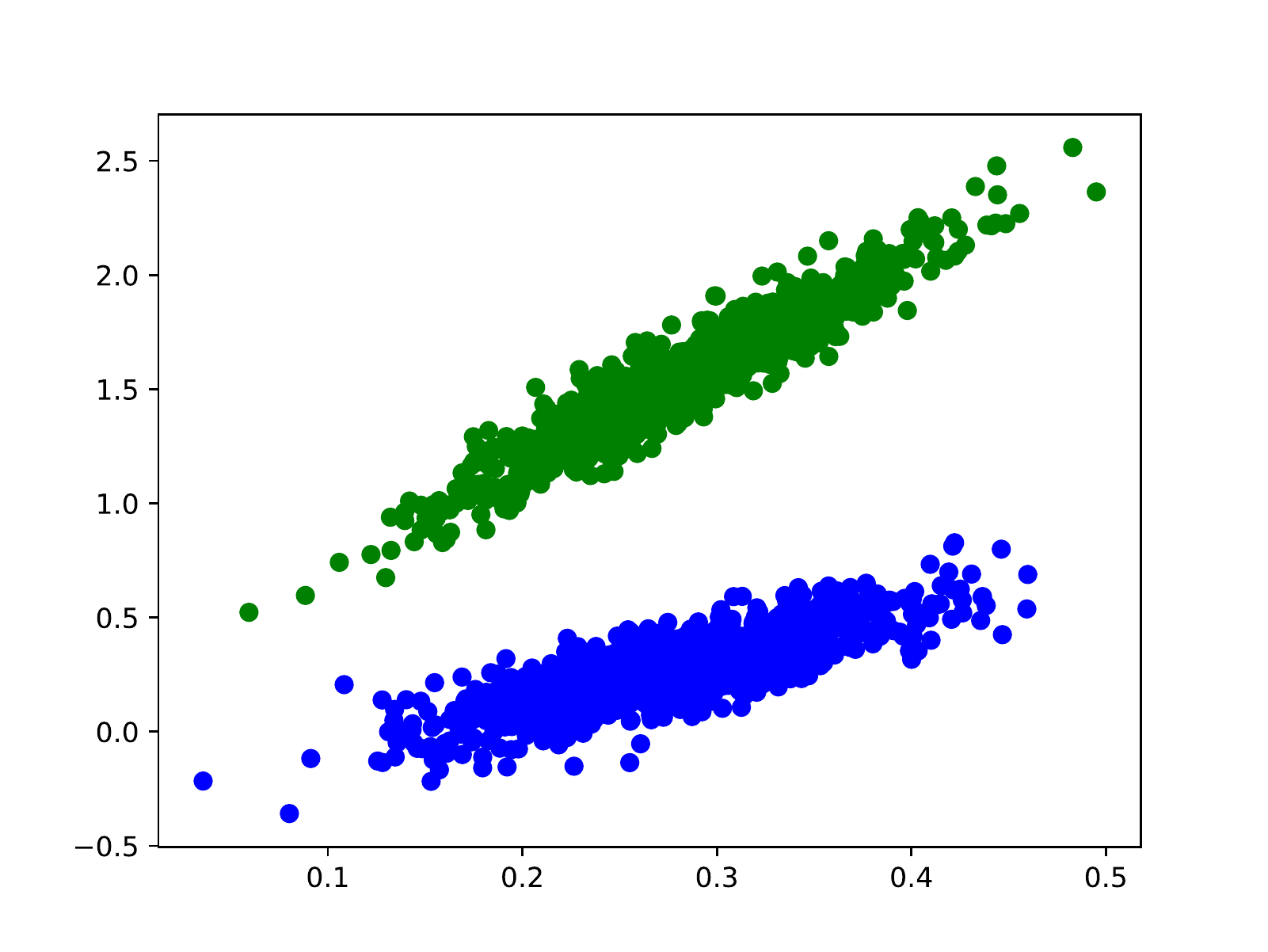}
    \caption{Encoding increases alignment}
    \end{subfigure}
    \caption{Embedding data from a Gaussian mixture via a linear transformation can lead to alignment of the major axes depending on the spectrum of the transformation.}
    \label{fig:s3_alignment}
\end{figure}



We briefly provide intuition around how stretching data, $X$, via multiplication by a linear transformation, $B$, can lead to alignment.  We begin by formally defining alignment:  we say that two directions $q_1, q_2$ become more \textit{aligned} in the latent space if their cosine similarity (which we write as $\cos(q_1, q_2)$) is smaller than their cosine similarity after embedding, $\cos(Bq_1, Bq_2)$.  Naturally, the alignment of directions in the latent space depends on $B$, as is illustrated by the following lemma (proof in Appendix \ref{appendix: Alignment Intuition}). 

\begin{lemma}
\label{lemma: Alignment Intuition}
    Let $B \in \mathbb{R}^{h \times d}$ have singular value decomposition $B=U\Sigma_BV^T$, where $\sigma_i$ and $v_i$ are the i\textsuperscript{th} singular value and right singular vector of $B$. Let $q_1, q_2 \in \mathbb{R}^d$ be unit vectors. Then,
    $$\cos(Bq_1, Bq_2) = \frac{
        \sum_{i=1}^d \sigma_i^2 \cos(q_1, v_i) \cos(q_2, v_i)
        }{\sqrt{
            \left(\sum_{i=1}^d \sigma_i^2 (\cos(q_1, v_i) ^2\right)
            \left(\sum_{i=1}^d \sigma_i^2 (\cos(q_2, v_i)^2\right)
        }}.$$
\end{lemma}

Lemma \ref{lemma: Alignment Intuition} demonstrates that the spectrum of $B$ governs whether two directions become more aligned.  For example, if $\sigma_1$ is much larger than $\sigma_i$ for $i > 1$ and the top eigenvector of $B$ has roughly similar cosine similarity between the two directions (i.e $\cos (q_1, v_1) \approx \cos(q_2, v_1)$), then $q_1$ and $q_2$ will become more aligned in the latent space.  In Figure \ref{fig:s3_alignment}, we give an example of a matrix $B$ that aligns the major axes of multivariate Gaussians by stretching along the direction $[1, 1]^T$.  

In this work, we identify a class of matrices $B$ that provably boost alignment for directions of interest.  We begin with a characterization of such $B$ for Gaussian mixture models in the next section.

\section{Stretching Along Left Singular Vectors Leads to Alignment}
\label{sec: Stretching}
In this section, we will show that stretching along the top left singular vectors of the data can boost alignment along directions of interest.  We begin by using the power method \cite{TrefethenBau} to provide intuition as to why stretching along singular vectors of the data leads to alignment.  

\textbf{Intuition from the Power Method.} The power method \cite{TrefethenBau} provides a mechanism for producing trivial alignment of directions in data.  Namely, repeated multiplication of data $X  \in \mathbb{R}^{d \times n}$ by a polynomial in $XX^T$ will project the data onto the top eigenvector of $XX^T$.  In this setting, the embedded points will lie on a line and thus, the cosine similarity between the difference of pairs of unique points will be either $1$ (perfectly aligned) or $-1$ (perfectly anti-aligned).\footnote{The number of iterations needed to produce perfect alignment depends on the numerical precision.}


Fully projecting the data onto this top eigenvector is undesirable since this projection is not invertible and thus, we cannot decode the embedding back to the original input space. However, a partial iteration of the power method using $XX^T$ will stretch along the top eigenvector of $XX^T$ while yielding an invertible map.  We will demonstrate that using the power method to partially stretch along this top eigenvector is beneficial for aligning meaningful directions.  We begin by proving that stretching along this top eigenvector aligns the major axes of Gaussians in a mixture model.

\subsection{Identifying the Top Singular Vector of a Gaussian Mixture}

We will first characterize the top eigenvector and eigenvalue of $XX^T$ drawn from a mixture of Gaussians, and then use this result to quantify the increase in alignment after stretching along this top eigenvector. To aid our analysis, we choose a simple setting where we can explicitly compute the eigenvectors and eigenvalues of $\frac{1}{n}XX^T$ as $n \to \infty$ (Proof in Appendix~\ref{appendix: stretching_lemma}).

\begin{lemma}
\label{lemma: Stretching along means}

Let $X_j = \{x_i^{(j)}\}_{i=1}^{n}  \in \mathbb{R}^{d \times n}$ for $j \in [m]$, where $x_i^{(j)} \sim \mathcal{N}(\mu_j, \Sigma_j)$. Assume $\Sigma_j = D_j + \mu_j(\sum_{i\neq j} \mu_i)^T$, such that $\sum_{j=1}^m D_j = cI$. Then as $n \to \infty$, 
$$\frac{1}{n} XX^T \xrightarrow{a.s} \EE\left[\frac{1}{n}XX^T\right] = cI + \left(\sum_{j=1}^m \mu_j\right)\left(\sum_{j=1}^m \mu_j\right)^T $$
with top eigenvector $(\sum_{j=1}^m \mu_j)/||\sum_{j=1}^m \mu_j||_2$ with eigenvalue $c + ||(\sum_{j=1}^m \mu_j)||_2^2$ and remaining eigenvectors $\{u ~; ~ ||u||_2^2 = 1, u \perp (\sum_{j=1}^m \mu_j)\}$ with eigenvalue $c$. 
\end{lemma}

By Lemma \ref{lemma: Stretching along means}, multiplication by $\frac{1}{n}XX^T$  will stretch along the direction of the average of the means.  We next verify that stretching along this direction leads to an alignment of major axes in a mixture of Gaussians.  



\subsection{Quantifying Alignment in Gaussian Mixture Models}

We now consider a setting with two Gaussians in $\mathbb{R}^2$ and use Lemma \ref{lemma: Stretching along means} to prove that stretching this mixture of Gaussians along the average of the means will align the major axes of the individual Gaussians (proof in Appendix \ref{appendix: two gaussians}).
\begin{corollary}
\label{corollary: two gaussians}
Let $X_1, X_2 \in \mathbb{R}^{2 \times n}$ with $X_1^{(i)} \sim \mathcal{N}(e_1, \Sigma_1), X_2^{(i)} \sim \mathcal{N}(e_2, \Sigma_2)$ for $i \in [n]$ where $e_1=[1,0]^T$, $e_2=[0,1]^T$, and $J$ is the all ones matrix.  Let $X = [ X_1 | X_2 ]$, and
$$D_1 = \begin{bmatrix} c+1 & 0 \\ 0 & 1 \end{bmatrix}, \Sigma_1 = D_1 + J ~~ ; ~~ D_2 = \begin{bmatrix} 1 & 0 \\ 0 & c+1 \end{bmatrix}, \Sigma_2 = D_2 + J.$$
Then $\EE[XX^T]$ has eigenvectors $u_1=\frac{1}{\sqrt{2}}[1,1]^T$ and  $u_2=\frac{1}{\sqrt{2}}[1,-1]^T$.

Consider the embedding $BX$, where $B \in \mathbb{R}^{d \times d}$ has eigenvectors $u_1$ and $u_2$ and corresponding eigenvalues $\alpha$ and $\beta$ where $\alpha^2 > \beta^2$. Let $q_{\Sigma_1}$ and $q_{\Sigma_2}$ be the top eigenvectors of $\Sigma_1$ and $\Sigma_2$. Then the cosine similarity of $q_{\Sigma_1}$ and $q_{\Sigma_2}$ before and after embedding with $B$ is:
$$ \cos(q_{\Sigma_1}, q_{\Sigma_2}) = \frac{2}{\sqrt{4 + c^2}}, \quad \quad  \cos(Bq_{\Sigma_1}, Bq_{\Sigma_2}) = \frac{2 + \gamma \sqrt{c^2 + 4}}{2\gamma + \sqrt{c^2 + 4}}$$
for $\gamma = \frac{(\alpha^2 - \beta^2)}{(\alpha^2 + \beta^2)}$.
Moreover, $\cos(Bq_{\Sigma_1}, Bq_{\Sigma_2})>  \cos(q_{\Sigma_1}, q_{\Sigma_2}).$
\end{corollary}



Thus, if $B$ stretches along the top left singular vectors of $X$, then the major axes of individual Gaussians in the mixture become more aligned.  As is shown by the following corollary, setting $B = (\frac{1}{n} XX^T)^k$, for some degree $k \geq 1$ leads to alignment (Proof in Appendix \ref{appendis: two gaussians with xxt}).


\begin{corollary}
\label{corollary: two gaussians with xxt}
In the setting of Corollary \ref{corollary: two gaussians}, if $B = (\frac{1}{n}XX^T)^k$,  then as $n$ tends to infinity,
$$\cos(B q_{\Sigma_1}, B q_{\Sigma_2}) \xrightarrow{a.s} \frac{2 + \gamma \sqrt{c^2 + 4}}{2\gamma + \sqrt{c^2 + 4}} , \quad \gamma = \frac{(\alpha^2 - \beta^2)}{(\alpha^2 + \beta^2)}$$
for $\alpha = (c+7)^k, \beta = (c+3)^k$.  Moreover, as $k$ increases $\cos(B q_{\Sigma_1}, B q_{\Sigma_2})$ also increases. 
\end{corollary}
\textbf{Remarks.} We empirically verify this result in Appendix \ref{appendis: two gaussians with xxt}. The above corollaries imply that stretching along the top singular vector of the data aligns the major axes of the individual Gaussians.  In general, however, the degree of alignment depends on how aligned each direction of interest is to the overall stretching direction. As a simple example, in the above setting, stretching instead along the major axis of one of the Gaussians would not significantly boost alignment. Indeed, the ideal direction, $v$, to stretch along is one where semantically meaningful directions are roughly equally aligned with $v$.   Consistent with the result of this section, we show in Section \ref{sec: Applications} that stretching along the top eigenvector of $XX^T$ leads to alignment of semantically meaningful directions on large datasets.

\section{Alignment with Over-parameterized Autoencoders}
\label{sec: Autoencoders}
In this section, we prove that over-parameterized, 1-hidden layer linear autoencoders trained using gradient flow from an asymmetric spectral initialization produce embeddings that are stretched along the top left singular vectors of the data.  We then prove that a similar phenomenon occurs for non-linear 1-hidden layer ReLU autoencoders trained on two data points.  We lastly connect this stretching to alignment of directions of interest in the latent space by demonstrating that asymmetrically initialized autoencoders align the major axes of a Gaussians in the mixture model from Section \ref{sec: Stretching}.  We begin with a review of autoencoder preliminaries. 

\subsection{Autoencoder Preliminaries}
\label{sec: Preliminaries}

We begin with the training setup for vanilla autoencoders.  

\begin{definition}[Autoencoding]  Given a dataset $X \in \mathbb{R}^{d \times n}$ with $n$ training examples and neural network $f_{\theta}(x): \mathbb{R}^{d} \to \mathbb{R}^{d}$ with parameters $\theta \in \mathbb{R}^{p}$, gradient descent is used to solve
\begin{align}
    \label{eq: Autoencoder Loss} 
    \argmin_{\theta \in \mathbb{R}^{p}}  \mathcal{L}(X, \theta), ~~ \text{where} ~~ \mathcal{L}(X, \theta) =  \sum_{i=1}^{n} \| x^{(i)} - f(x^{(i)})\|_2^2, ~ 
\end{align}
 where $x^{(i)}$ denotes the $i$\textsuperscript{th} column of $X$.  
\end{definition}

In this work, we will reference the following spectral initialization scheme used in \cite{ImplictRegularizationBach, Saxe2014, Saxe2019}. 
\begin{definition}[Spectral Initialization]  
\label{def: Spectral Initialization}
Let $f(x) = A B x$ denote a linear autoencoder with $x \in \mathbb{R}^{d}$, $A \in \mathbb{R}^{d \times k}$, and $B \in \mathbb{R}^{k \times d}$.  We assume gradient descent is used to solve Eq. \eqref{eq: Autoencoder Loss} on dataset $X$ and let $A^{(0)}, B^{(0)}$ denote the parameters at initialization.  The \textbf{spectral initialization} corresponds to $A^{(0)} = U_X \Sigma_A^{(0)} U_X^T,  B^{(0)} = U_X \Sigma_B^{(0)} U_X^T$ where $U_X$ denotes the left singular vectors of $X$ and $\Sigma_A^{(0)}, \Sigma_B^{(0)}$ are positive, semi-definite diagonal matrices.   
\end{definition}

The spectral initialization scheme is commonly used to understand training dynamics in neural networks.  Assuming \textit{symmetric} spectral initialization (i.e. $\Sigma_A^{(0)} = \Sigma_B^{(0)}$)  and that the rank of $A^{(0)}$ equals the rank of $X$, solving Eq. \eqref{eq: Autoencoder Loss} with gradient flow (i.e. gradient descent with infinitesimal step size) recovers the well-known result that linear autoencoders project onto the subspace spanned by the top principal components of the data \cite{AutoencodersPCAVectors}.  



\subsection{Stretching in Over-parameterized Linear Autoencoders}
We begin by deriving the gradient flow solution for 1-hidden layer autoencoders from arbitrary spectral initialization \cite{ImplictRegularizationBach, Saxe2014, Saxe2019}.  By exploring how initialization impacts the singular values of the gradient flow solution, we  derive a simple, computationally tractable asymmetric spectral initialization scheme that yields stretching along the top left singular vectors of the training data. 

\begin{theorem}
\label{theorem: two layer linear autoencoder}
Let $f(x) = ABx$ denote a neural network with $A, B \in \mathbb{R}^{d \times d}$ and $x \in \mathbb{R}^{d}$, and let $X \in \mathbb{R}^{d \times n}$ denote the training data.  If $A, B$ are initialized according to the spectral initialization (Definition \ref{def: Spectral Initialization}), then gradient flow on $A, B$ converges to: 
\begin{align*}
    A^{(\infty)} &= U_X \Sigma_A^{(\infty)} U_X^T ~~;~~ B^{(\infty)} = U_X \Sigma_B^{(\infty)} U_X^T,  ~~  \\
    {\Sigma_{B}}_{i,i}^{(\infty)} &= \sqrt{\frac{({\Sigma_B^{(0)}}_{i,i}^2 - {\Sigma_A^{(0)}}_{i,i}^2)}{2} + \sqrt{1 + \frac{({\Sigma_B^{(0)}}_{i,i}^2 - {\Sigma_A^{(0)}}_{i,i}^2)^2}{4}}} ~~ ; ~~ {\Sigma_{A}^{(0)}}_{i,i}^{(\infty)} = {\Sigma_{B}^{(0)}}_{i,i}^{-1}.
\end{align*}
\end{theorem}

The full proof is provided in Appendix \ref{appendix: Gradient flow solution linear autoencoders}; it is inspired by the proof strategy used in \cite{RadhakrishnanOPAMemory}.  We now provide an initialization scheme that leads to alignment in the latent space.   


\noindent \textbf{Asymmetric Spectral Initialization.} In order to generate aligned embeddings, we now turn to the following asymmetric spectral initialization.  Provided that $B^{(0)}$ is a polynomial in $\frac{1}{n}XX^T + I$ and $A^{(0)} = \mathbf{0}$, Theorem \ref{theorem: two layer linear autoencoder} guarantees that $B$ stretches inputs along the top directions given by the top left singular vectors of the training data, $X$.  This is formalized in the following corollary.

\begin{corollary}
\label{corollary: Polynomial Stretching} 
Let $f(x) = ABx$ denote a neural network with $A, B \in \mathbb{R}^{d \times d}$ and $x \in \mathbb{R}^{d}$, and let $U_X \Sigma_X V_X^T$ denote the SVD of $X$. Let $A^{(0)} = \mathbf{0}$ and $B^{(0)}$  be a polynomial in $\frac{1}{n}XX^T + I$.  Let $\sigma_i(B^{(t)})$ denote the $i$\textsuperscript{th} singular value of $B^{(t)}$.  Then $B^{(\infty)}$ given by gradient flow satisfies: 
\begin{align*}
   \sigma_i(B^{(\infty)}) \geq \sigma_i(B^{(0)}) \geq 1.
\end{align*}
\end{corollary}

\textbf{Remarks.}  
 From Corollary \ref{corollary: Polynomial Stretching}, we find that $B^{(\infty)}$ after gradient descent from an asymmetric spectral initialization will stretch along the top singular vectors even more than $B^{(0)}$. We note that this property does not necessarily hold for symmetric spectral initialization.  Moreover, we can achieve an arbitrary amount of stretching along the top singular vectors of the data by scaling the degree of the matrix polynomial used to initialize $B^{(0)}$.  Finally, unlike general spectral initialization, our asymmetric initialization above does not require computing any singular vectors, and hence is readily usable in practice.


\subsection{Aligning Latent Spaces with Linear Autoencoders}
Returning to the example of a Gaussian mixture from Section \ref{sec: Stretching}, we now use Theorem \ref{theorem: two layer linear autoencoder} to prove that asymmetrically initialized linear autoencoders can promote alignment. Specifically, we revisit the setting of Corollary \ref{corollary: two gaussians}, which consisted of two Gaussians whose major axes are nearly orthogonal. By generating $B$ from our linear autoencoder, we find that autoencoding will align the major axes of these Gaussians (Proof in Appendix \ref{appendix: two gaussians after autoencoding}).
\begin{corollary}
\label{corollary: two gaussians after autoencoding}
In the setting of Corollary \ref{corollary: two gaussians}, set $B$ to be the encoding matrix of an autoencoder $f(x)=ABx$ after gradient flow where $A$ is initialized to $0$ and $B$ is initialized to $(\frac{1}{n}XX^T)^k$. Then, as n tends to infinity,
$$\cos(Bq_{\Sigma_1}, Bq_{\Sigma_2}) = \frac{2 + \gamma \sqrt{c^2 + 4}}{2\gamma + \sqrt{c^2 + 4}} , \quad \gamma = \frac{(\alpha^2 - \beta^2)}{(\alpha^2 + \beta^2)},$$
where
$$\alpha = \lambda_1(B) = \sqrt{\frac{(c+7)^{2k}}{2} + \sqrt{1 + \frac{(c+7)^{4k}}{4}}}, \;\beta = \lambda_2(B) = \sqrt{\frac{(c+3)^{2k}}{2} + \sqrt{1 + \frac{(c+3)^{4k}}{4}}}.$$
Thus the alignment of $q_1$ and $q_2$ will increase, and this alignment can be arbitrarily boosted by encoding with higher degree $k$.
\end{corollary}

Since autoencoders with asymmetrical spectral initialization stretch along the top directions of the dataset, we can use them to boost  alignment while allowing for perfect reconstruction via decoding. In Section \ref{sec: Applications}, we explore how linear autoencoders can be used to align real-world datasets. 



\subsection{Stretching in Over-parameterized Nonlinear ReLU Autoencoders}
In the previous section, we proved that asymmetrically initialized linear over-parameterized autoencoders stretch along the top left singular vectors of the data.  We now provide a proof that the same is true for asymmetrically initialized over-parameterized ReLU autoencoders in a simplified setting with two opposing datapoints. 
We use gradient flow to compute the closed forms of  $A^{(\infty)}, B^{(\infty)}$. Since the dataset $[-x, x]$ has exactly one singular vector with non-zero singular value (i.e $x$), we compute the stretching factor as the magnitude of the new embedding of $x$ (i.e $||Bx||_2$).

\begin{theorem}
Let $f(x) = A\phi(Bx)$ be a neural network with $A \in \mathbb{R}^{d \times 2d}, B \in \mathbb{R}^{2d \times d}$, and $\phi$ as a ReLU activation. For some $x \in \mathbb{R}^{d}$ with $\eta = ||x||_2^2$, let $X = \{x, -x\}$ denote the training set. 
Let $A$ and $B$ be initialized as
$$A^{(0)} = 0, \;B^{(0)} = \begin{bmatrix} xx^T \\ -xx^T\end{bmatrix}.$$
Then after using gradient flow, we have that $A^{(\infty)} = \alpha  (B^{(0)})^T$ and  $B^{(\infty)} = \beta B^{(0)}$, where
$$\beta^2 = \frac{1}{2} \left(\frac{\sqrt{\eta^4 + 4}}{\eta^2} +1\right), \;\alpha \beta = \frac{1}{\eta}.$$
Moreover, gradient flow will stretch the embedding of $x$ by a quadratic factor of $\eta$:
$$ ||Bx||_2^2 = \eta^2 \left(\frac{\sqrt{\eta^4 + 4}}{\eta^2} +1\right) \|x\|^2.$$
\label{theorem:non_linear}
\end{theorem}
The full proof can be found in Appendix \ref{appendix: non_linear}, but proceeds through gradient flow analysis similar to Theorem \ref{theorem: two layer linear autoencoder}. As in the linear setting, we find that by asymmetrically initializing $A$ and $B$, we can stretch the embedding of the dataset along the top left singular vector.

\section{Applications}
\label{sec: Applications}
Thus far, we have demonstrated that autoencoders can align directions of interest in the latent space by stretching along the left singular vectors of the data.  While we primarily focused on proving this for the case of a Gaussian mixture model, we now apply our theoretical results to produce autoencoders that align semantic/syntactic directions in word embedding space and drug signatures in gene expression space.  

\subsection{Alignment of GloVe Embeddings}
\begin{figure}[t]
    \centering
    \begin{subfigure}{0.45\textwidth}
     \includegraphics[width=\textwidth]{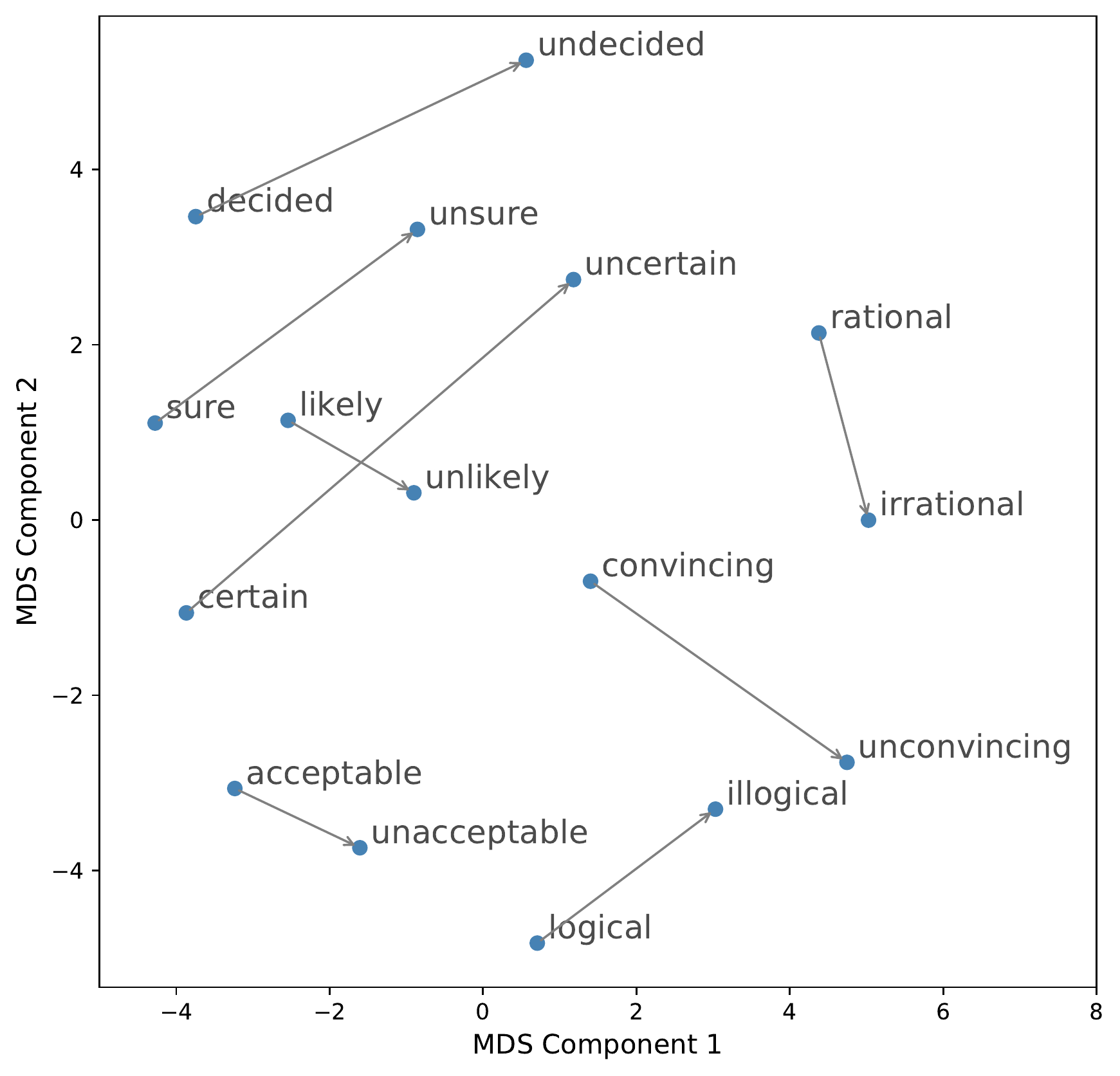}
     \caption{Original Space}
    \end{subfigure}
    \begin{subfigure}{0.45\textwidth}
     \includegraphics[width=\textwidth]{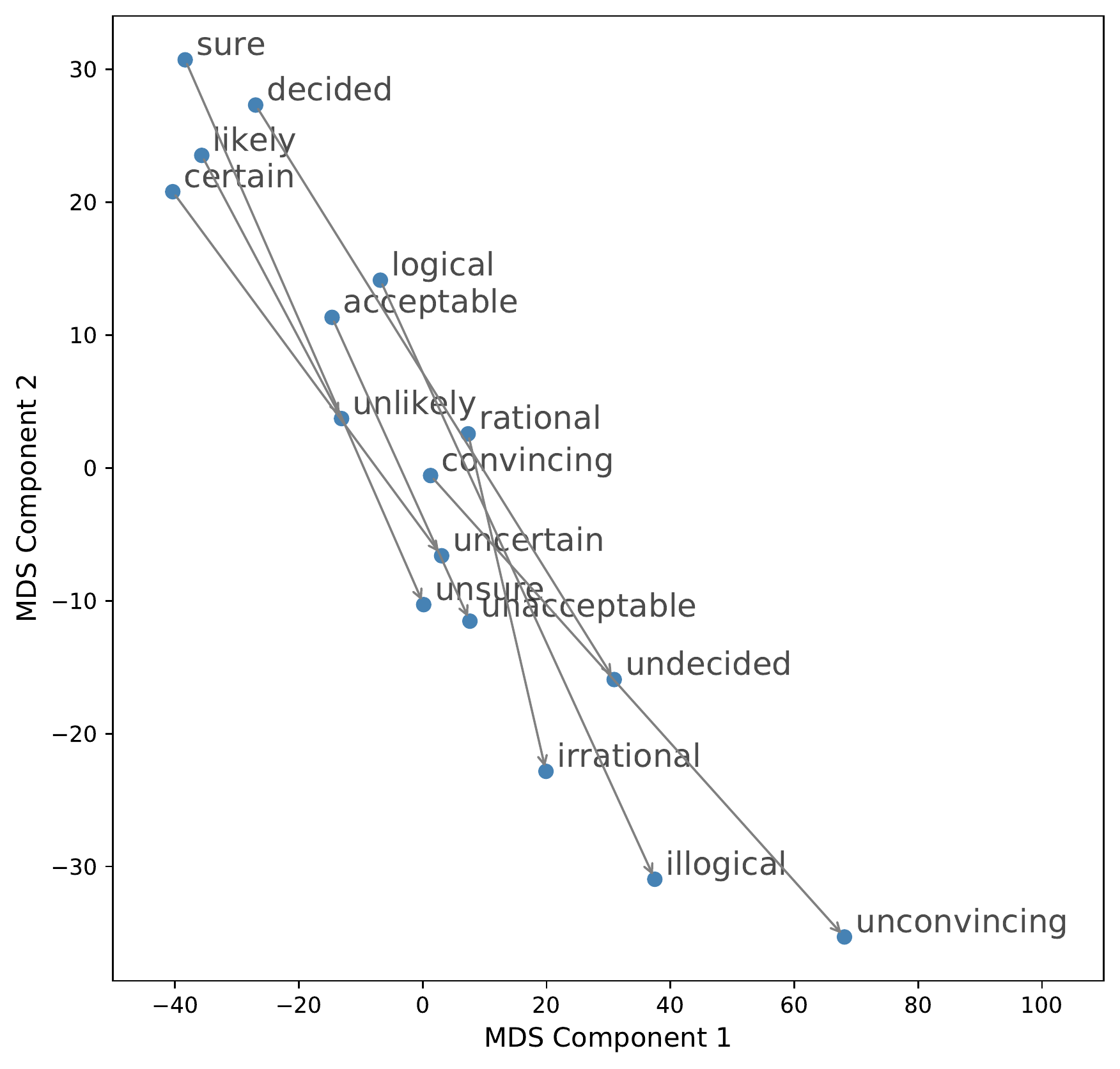}
     \caption{Linear Autoencoder (k=2)}
    \end{subfigure}
    \caption{We visualize the embeddings of several pairs of words with a negation relationship using MDS both in the original space and linearly autoencoded space (k=2). The linearly autoencoded space aligns the word pairs to point in the same direction.}
    \label{fig:negation_alignment}
\end{figure}

In this section, we apply our theory to align semantic/syntactic directions in GloVe word embeddings \cite{pennington2014glove}. We use 300 dimensional GloVe vectors that were trained on Common Crawl with 840 billion tokens. We consider the embeddings of the 500,000 most frequent words, and encode these embeddings with linear and non-linear autoencoders. Details about our experimental setup and training procedures can be found in Appendix \ref{appendix: experimental_setup}.

\subsubsection{Alignment of Meaningful Directions}
We first consider how autoencoders can align meaningful directions in the GloVe embedding data. We consider word relationships from the Google analogy dataset \cite{mikolov2013linguistic}, which consists of 19,544 question pairs over 14 types of relationships. Such relationships include semantic directions (such as man $\rightarrow$ woman) and syntactic directions (such as cat $\rightarrow$ cats). 

We study how stretching the top singular vectors of the data through autoencoders can align these directions. Prior work suggests that the top singular vectors of GloVe data are closely related to the frequency of a word \cite{mu2017all}. We confirm this finding by plotting, for word embedding $a$, the frequency ranking of that word against its cosine similarity with the singular vector $u_1$ (Appendix \ref{appendix:glove_frequency}). Recall from Lemma \ref{lemma: Alignment Intuition} that, to boost alignment, we would like the the stretching axis to be roughly equally aligned with the directions of interest. Since we consider relationships between words which are independent of word frequency, we suspect that stretching along the top singular vectors of the GloVe will allow us to align word relationships in the data. 

To evaluate alignment in these directions, we randomly draw two word pairs $(a_1, b_1)$ and $(a_2, b_2)$ which capture the same relationship.  We then plot the cosine distances of $b_1 - a_1$ and $b_2 - a_2$ before and after autoencoding (Appendix \ref{appendix: glove_alignment}). We find that autoencoding with either a linear or nonlinear autoencoder boosts alignment of these semantic directions in the latent space.

We can further visualize this alignment by visualizing the latent space before and after autoencoding with multidimensional scaling (MDS). In Figure \ref{fig:negation_alignment}, we visualize the alignment of a subset of the pairs in the relationship ``negation" (e.g aware $\rightarrow$ unaware). We find that autoencoding with both linear and non-linear autoencoders is able to align the direction of negation. We depict the full MDS visualization for negation and demonstrate alignment of other relationships in Appendix \ref{appendix:glove_mds}.

\renewcommand{\arraystretch}{0.8}
\begin{table}[t]
    \centering
    \caption{Analogy and word similarity metrics for GloVe data before and after autoencoding. Autoencoding can improve analogy scores without causing a significant degradation in word similarity. }
    \begin{tabular}{l|cc|cc}
    \toprule
        Model & \begin{tabular}{c}\texttt{wsim}\\(MTurk-771)\end{tabular} & \begin{tabular}{c}\texttt{wsim}\\(MEN)\end{tabular}& \begin{tabular}{c}\texttt{analogy}\\(Google)\end{tabular} & \begin{tabular}{c}\texttt{analogy-accuracy}\\(Google)\end{tabular}  \\\midrule
        Original  &   \textbf{0.629} &  \textbf{0.739} & 0.417 & 0.338\\
        Linear AE (k=1) &  0.628 & 0.738 & 0.512 &\textbf{0.382}\\
        Linear AE (k=2) & 0.611 & 0.729 & \textbf{0.656} & 0.322 \\
        Leaky ReLU AE & 0.592 &  0.698 &  0.509 & 0.370\\
        \bottomrule
    \end{tabular}
    \label{tab:glove_metrics}
\end{table}

\subsubsection{Word Analogy Metrics}
To evaluate alignment on a more quantitative level, we study the effect of autoencoding on word analogy tasks. We consider a list of 4-tuples of the form $(a_1, b_1, a_2, b_2)$ which have the analogy relationship $a_1$:$b_1$::$a_2$:$b_2$. One such example might be the embeddings of (men, women, king, queen). 

We evaluate two metrics for word analogy alignment. In the first, \texttt{analogy}, we simply calculate the average cosine similarity $\cos(a_2-a_1, b_2-b_1)$. In the second, \texttt{analogy-accuracy}, we evaluate whether $b_2 = \text{argmax}_{b^*} \cos(b^*, a_2 - a_1 + b_1)$ following \cite{levy2014linguistic}. We form the 4-tuples by combining every possible pair within each relationship in the Google Analogy dataset \cite{mikolov2013linguistic}. 

In Table \ref{tab:glove_metrics}, we find that autoencoding substantially boosts the  \texttt{analogy} metric, as semantic directions become more aligned. Intriguingly, we can boost alignment beyond a nonlinear autoencoder by simply using a linear autoencoder with a higher degree $k$. Moreover, we find that encoding the GloVe dataset with a smaller degree linear autoencoder also boosts accuracy on the analogy task.

\subsection{Word Similarities}
In the previous subsections, we showed that autoencoding word embeddings can align meaningful linguistic directions. Word locality is another important property of word embedding spaces:  similar words should have similar representations. In this section, we show that aligning the space through autoencoding does not significantly degrade the locality of words in the latent space. 

Word similarity datasets are comprised of word pairs along with human annotated ratings of the words' relatedness. We compute the similarity of two words as the cosine similarity between the words' embeddings. Following \cite{chung2017learning}, we evaluate the metric \texttt{wsim} as the Spearman's rank correlation between the rankings computed via cosine similarity and the rankings from the human annotation. 
We use two word-similarity datasets:  Mturk-771 \cite{halawi2012large}, which rates 771 word pairs on a scale of 0 to 5, and MEN \cite{bruni2014multimodal}, which rates 3000 word pairs on a scale of 0 to 50. 

\paragraph*{Reference Point Correction} Word similarity metrics typically compute the cosine similarity of individual embeddings with respect to the origin $o = \mathbf{0}$. However, since autoencoding stretches along the top singular vector, this origin is a poor reference point since it does not stretch along with the rest of the dataset. Instead, we choose a reference point that will also stretch alongside the top principal component: for the original GloVe dataset $X$, we set the reference point to be $o = \lambda_1 u_1$, where $\lambda_1, u_1$ are the top left singular value and vector of $X$. Further considerations on this correction can be found in Appendix \ref{appendix: reference_point_correction}. 

In Table \ref{tab:glove_metrics}, we find that the word similarity metrics remain largely stable after autoencoding, indicating that alignment does not significantly warp the latent space. Thus, autoencoding can align GloVe data without losing important properties such as word locality.

\subsection{Drug Signature Alignment in CMap}

Finally, we apply our theoretical results to gene expression data from CMap \cite{CMap} to boost the alignment of drug signatures, i.e the difference in gene expression for a cell type at control and after applying the drug. Prior work \cite{COVIDAutoencoding} demonstrated that 1-hidden layer over-parameterized autoencoders with Leaky ReLU \cite{LeakyReLU} activation produced pre-activation latent spaces which aligned drug signatures across cell types.  The boost in alignment was evaluated by plotting the cosine similarity of drug signatures across cell types in the original and the autoencoded latent space.  

Following the setting of \cite{COVIDAutoencoding}, we measure the alignment of drug signatures between MCF7 and A549 cell types using linear and nonlinear autoencoders (Figure \ref{fig: Drug Signature Alignment}). We find that linear autoencoders with asymmetric spectral initialization can boost alignment even beyond the alignment given by the Leaky ReLU autoencoder in \cite{COVIDAutoencoding}. Moreover, using a high degree polynomial for asymmetric initialization (i.e., degree 30 in Figure~\ref{fig: Drug Signature Alignment}c) leads to all drug signatures being nearly perfectly aligned.  



\begin{figure}
    \captionsetup[subfigure]{justification=centering}   
    \centering
    \begin{subfigure}{0.32\textwidth}
    \includegraphics[width=\textwidth]{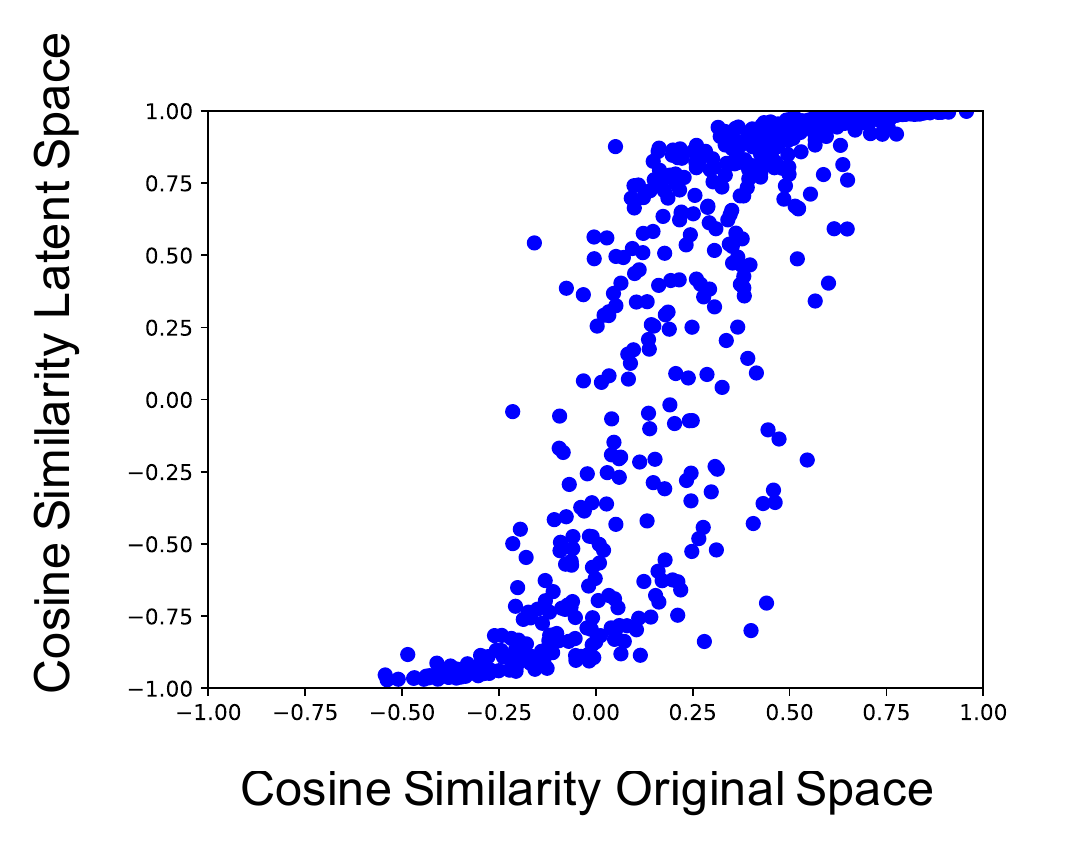}
    \centering
    \caption{Leaky ReLU Autoencoder}
    \end{subfigure}
    \begin{subfigure}{0.32\textwidth}
    \centering
    \includegraphics[width=\textwidth]{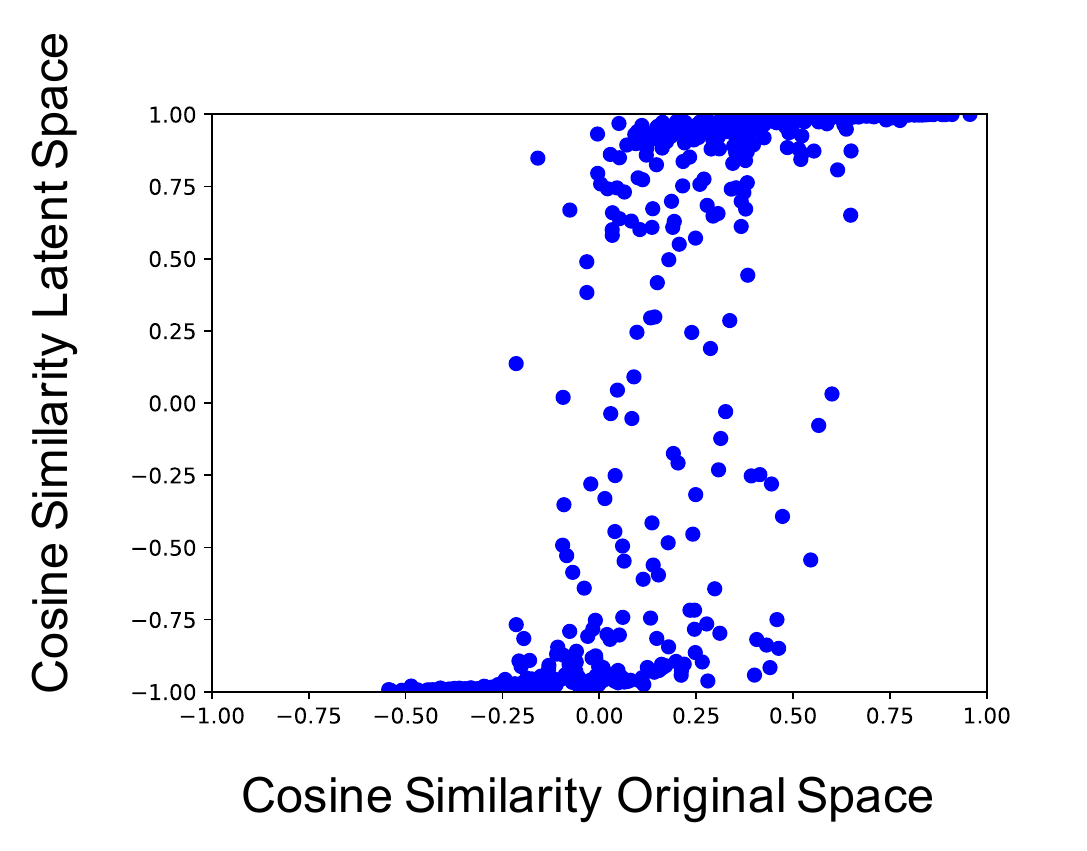}
    \caption{Linear Autoencoder (k=10)}
    \end{subfigure}
    \begin{subfigure}{0.32\textwidth}
    \centering
    \includegraphics[width=\textwidth]{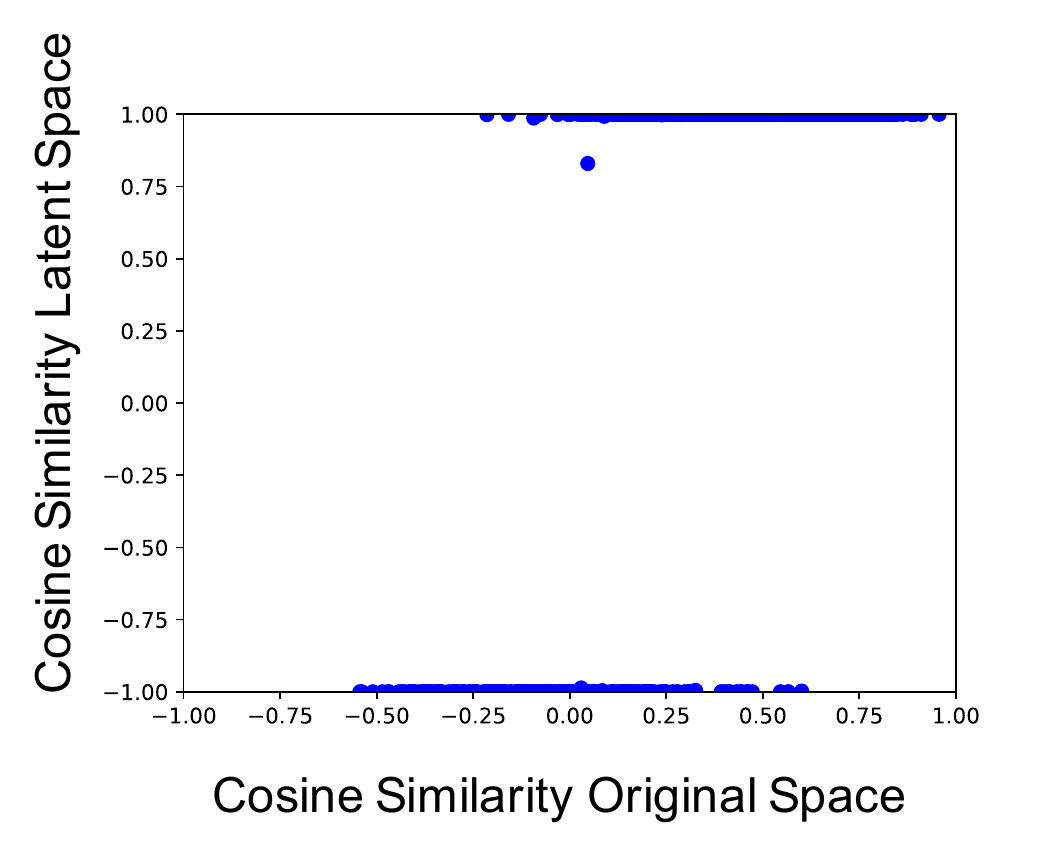}
    \caption{Linear Autoencoder (k=30)}
    \end{subfigure}
    \caption{Alignment of drug signatures across 2 cell types (MCF7 and A549) using (a) LeakyReLU autoencoders from \cite{COVIDAutoencoding} and (b,c) linear autoencoders with asymmetric spectral initialization. Using linear autoencoders can lead to greater alignment than using a LeakyReLU autoencoder.}
    \label{fig: Drug Signature Alignment}
\end{figure}


\section{Conclusion}
\label{sec: Conclusion}
In this work, we analyzed a mechanism through which autoencoders with asymmetric spectral initialization  produce aligned latent spaces by stretching along the top singular vectors of the dataset. After fully characterizing the degree of stretching in linear autoencoders, we showed that we can arbitrarily stretch along the top directions by adjusting the initialization.  We further replicated this behavior for nonlinear autoencoders in a simplified setting. We lastly applied our theoretical results to develop autoencoders that align linguistic shifts in word embedding spaces and drug signatures in gene expression data.  

 While we analyzed autoencoders initialized using asymmetric spectral initialization, \cite{COVIDAutoencoding} observed that even the default initialization scheme from PyTorch \cite{PyTorch} can empirically lead to alignment.  An interesting direction of future work would be to provide theoretical guarantees around what initialization regimes in linear and non-linear settings guarantee alignment of semantically meaningful directions after gradient descent. Another interesting avenue would be to characterize alignment of autoencoders in infinite width settings to facilitate analysis of more complex autoencoders. While neural tangent  kernel (NTK) analyses are not readily applicable for understanding latent space analysis, other parameterizations that admit feature learning (such as the $\mu P$ parameterization \cite{GregYangFeatureLearning}) may be more helpful for understanding alignment.


\section{Acknowledgements}
This work was partially supported by NSF (DMS-1651995), ONR (N00014-17-1-2147 and N00014-18-1-2765), the MIT-IBM Watson AI Lab, and a Simons Investigator Award to C. Uhler. The Titan Xp used for this research was donated by the NVIDIA Corporation.   

\bibliographystyle{plain}
\bibliography{references}


\newpage

\appendix
\section{Proof of Lemma \ref{lemma: Alignment Intuition}: Alignment of linearly embedded datasets}
\label{appendix: Alignment Intuition}
\begin{proof}
We note that:
\begin{align*}
    (Bq_1)^T(Bq_2) &= q_1^T B^T B q_2 \\
    &= q_1^T V \Sigma B^2 V^T q_2\\
    &= \sum_{i=1}^d \sigma_i^2 (q_1^Tv_1) (q_2^Tv_2)\\
    &= \sum_{i=1}^d \sigma_i^2 \cos \theta_i^{(1)} \cos \theta_i^{(2)}
    ||Bq_j||_2 = \sqrt{q_j^T B^T B q_j}\\
    &= \sqrt{\sum_{i=1}^d  \sigma_i^2 (\cos \theta_i^{(j)})^2}.
\end{align*}
We thus have:
$$\cos(Bq_1, Bq_2)= \frac{(Bq_1)^T Bq_2}{\|Bq_1\|_2\|B q_2\|_2} = \frac{
    \sum_{i=1}^d \sigma_i^2 \cos (q_1, v_i) \cos(q_2, v_i)
    }{\sqrt{
        \left(\sum_{i=1}^d \sigma_i^2 (\cos (q_1, v_i)^2\right)
        \left(\sum_{i=1}^d \sigma_i^2 (\cos (q_2, v_i)^2\right)
    }}.$$
\end{proof}

\newpage 

\section{Proof of Lemma \ref{lemma: Stretching along means}: Stretching Lemmas}
\label{appendix: stretching_lemma}
\begin{proof}
Let $x_i^{(j)}$ be the the ith example from $X_1$ and $x_i^{(2)}$ be the ith example from $X_2$. Then we note that 
$$\frac{1}{n}XX^T = \frac{1}{n}\sum_{i=1}^n \begin{bmatrix} x_i^{(1)} & x_i^{(2)} ... & x_i^{(m)}\end{bmatrix}  \begin{bmatrix} x_i^{(1)} & x_i^{(2)} ... & x_i^{(m)}\end{bmatrix}^T.$$

Let $x = \begin{bmatrix} x^{(1)} &  x^{(2)}...& x^{(m)}\end{bmatrix}, x\in \mathcal{R}^{d \times m}$ where $x_i \sim \mathcal{N}(\mu_i, \Sigma_i)$. Then by the strong law of large numbers, as n tends to infinity, we have that:
$$B^{(0)} = \frac{1}{n}XX^T \xrightarrow{a.s} \EE[xx^T].$$

Furthermore
\begin{align*}
    \EE[xx^T] &= \sum_{i=1}^m \EE[x^{(i)}x^{(i)^T}]\\
    &= \sum_{i=1}^m \Sigma_i + \mu_i\mu_i^T.
\end{align*}
Moreover $\EE[xx^T] = \EE[\frac{1}{n}XX^T]$. We can then plug in the relevant values for $\Sigma_i$:
\begin{align*}
     \EE[xx^T] &= \sum_{i=1}^m D_i + \mu_i \sum_{j=1}^m \mu_j^T\\
     &=  cI + \left(\sum_{j=1}^m \mu_j\right)\left(\sum_{j=1}^m \mu_j\right)^T.
\end{align*}

\end{proof}
\newpage

\section{Proof of Corollary \ref{corollary: two gaussians}: Alignment of 2D Gaussians}
\label{appendix: two gaussians}

\begin{proof}
For our particular setting, from Lemma \ref{lemma: Stretching along means}, we know that:
$$\EE[\frac{1}{n}XX^T] = I + 2 J = (c + 3) I + 4\hat{v}\hat{v}^T$$
with eigenvectors $\hat{v} = \frac{[1,1]^T}{\sqrt{2}}$ and $\hat{u} = [1, -1]^T/\sqrt{2}$ with eigenvalues:
$\lambda_1 = (c + 7), \lambda_2 = (c+3)$. 

Since $B$ has the same eigenvectors as $\EE[XX^T]$ we have that $B$ can be expressed as:
$$U_X = \begin{bmatrix} \hat{v} & \hat{u} \end{bmatrix} = \frac{1}{\sqrt{2}}\begin{bmatrix}1 & 1 \\ 1 & -1\end{bmatrix}, \quad \Sigma_B =\begin{bmatrix} \alpha &  0 \\ 0 & \beta \end{bmatrix}, \quad B = U_X \Sigma_B U_X^T.$$
Simplifying our expression for $B$ gives us the following form:
$$B = \frac{1}{2}(\alpha J + \beta \tilde{J}), \quad \text{for } \tilde{J} = \begin{bmatrix} 1 & -1 \\ -1 & 1 \end{bmatrix}.$$

We want to see how the angle between the principal directions of the original (centered) datasets changes after embedding. We note that the top unnormalized eigenvectors of the two centered datasets are:
$$q_{\Sigma_1} = \begin{bmatrix} 2 \\ c + \sqrt{c^2 + 4}\end{bmatrix},
q_{\Sigma_2} = \begin{bmatrix} 2 \\ -c + \sqrt{c^2 + 4}\end{bmatrix}.$$

The cosine distance between these two directions is:
$$ \cos(q_{\Sigma_1}, q_{\Sigma_2}) =  \frac{2}{\sqrt{c^2 + 4}}.$$

The cosine similarity of the expected embeddings of these two directions is (after a fair amount of arithmetic):
\begin{align*}
   \cos(Bq_{\Sigma_1}, Bq_{\Sigma_2}) &= \frac{2 + \gamma \sqrt{c^2 + 4}}{2\gamma + \sqrt{c^2 + 4}} , \quad \gamma = \frac{(\alpha^2 - \beta^2)}{(\alpha^2 + \beta^2)}.
\end{align*}

Moreover note that since $\gamma > 0$, we have that 
$$ \cos(Bq_{\Sigma_1}, Bq_{\Sigma_2})  > \cos(q_{\Sigma_1}, q_{\Sigma_2}).$$ 
\end{proof}



\newpage 

\section{Proof of Corollary \ref{corollary: two gaussians with xxt}: Alignment of 2D Gaussians using $B=XX^T$}
\label{appendis: two gaussians with xxt}
\begin{proof}
Let $\bar{B} = \EE[\frac{1}{n}XX^T]^k$. From Lemma  \ref{lemma: Stretching along means}, we know that 
$$\EE[\frac{1}{n}XX^T] = I + 2 J = (c + 3) I + 4\hat{v}\hat{v}^T$$
with eigenvectors $\hat{v} = \frac{[1,1]^T}{\sqrt{2}}$ and $\hat{u} = [1, -1]^T/\sqrt{2}$ with eigenvalues:
$\lambda_1 = (c + 7), \lambda_2 = (c+3)$.
Then $\bar{B} = \EE[\frac{1}{n}XX^T]^k$ has the exact same eigenvectors, but has eigenvalues $\lambda_1 = (c+7)^k$ and $\lambda_2 = (c+3)^k$.

From Corollary \ref{corollary: two gaussians}, we know that the cosine distance of $q_{\Sigma_1}$ and $q_{\Sigma_2}$ after embedding with $\bar{B}$ will be  
$$\cos(\bar{B}q_{\Sigma_1}, \bar{B}q_{\Sigma_2})  = \frac{2 + \gamma \sqrt{c^2 + 4}}{2\gamma + \sqrt{c^2 + 4}} , \quad \gamma = \frac{(\alpha^2 - \beta^2)}{(\alpha^2 + \beta^2)}$$
for $\alpha = (c+7)^k, \beta = (c+3)^k$.

From Lemma  \ref{lemma: Stretching along means}, we know that $\frac{1}{n}XX^T \xrightarrow{a.s} \EE[\frac{1}{n}XX^T]$. Then, by the continuous mapping theorem, we have that $B=(\frac{1}{n}XX^T)^k \xrightarrow{a.s} (\EE[\frac{1}{n}XX^T])^k$. Therefore $B$ converges a.s to $\bar{B}$. Since computing the cosine similarity between $Bq_{\Sigma_1}$ and $Bq_{\Sigma_2}$ given $B$ is a continuous mapping, we have that 
$$\cos(Bq_{\Sigma_1}, Bq_{\Sigma_2}) \xrightarrow{a.s} \cos(\bar{B}q_{\Sigma_1}, \bar{B}q_{\Sigma_2}) = \frac{2 + \gamma \sqrt{c^2 + 4}}{2\gamma + \sqrt{c^2 + 4}} , \quad \gamma = \frac{(\alpha^2 - \beta^2)}{(\alpha^2 + \beta^2)}$$
for $\alpha = (c+7)^k, \beta = (c+3)^k$.

Moreover, as $k$ increases, $\gamma$ also increases, which means $\cos(Bq_{\Sigma_1}, Bq_{\Sigma_2})$ will go up and the two directions will become more aligned. 
We empirically verify this claim with several values of $c$, and find that our theoretical result holds up to the third decimal place.
\end{proof}

An empirical verifiction of Corollary \ref{corollary: two gaussians with xxt} is presented in Figure \ref{fig: Verification of Translation Alignment}.  For generating these plots, we chose $c = -.9$ and $k=35$\footnote{We used a scale factor of $\frac{1}{2.25}$ to control the norm of the latent dimensions for visualization.}.  

\begin{figure}
    \centering
    \begin{subfigure}{0.48\textwidth}
    \centering
    \includegraphics[width=\textwidth]{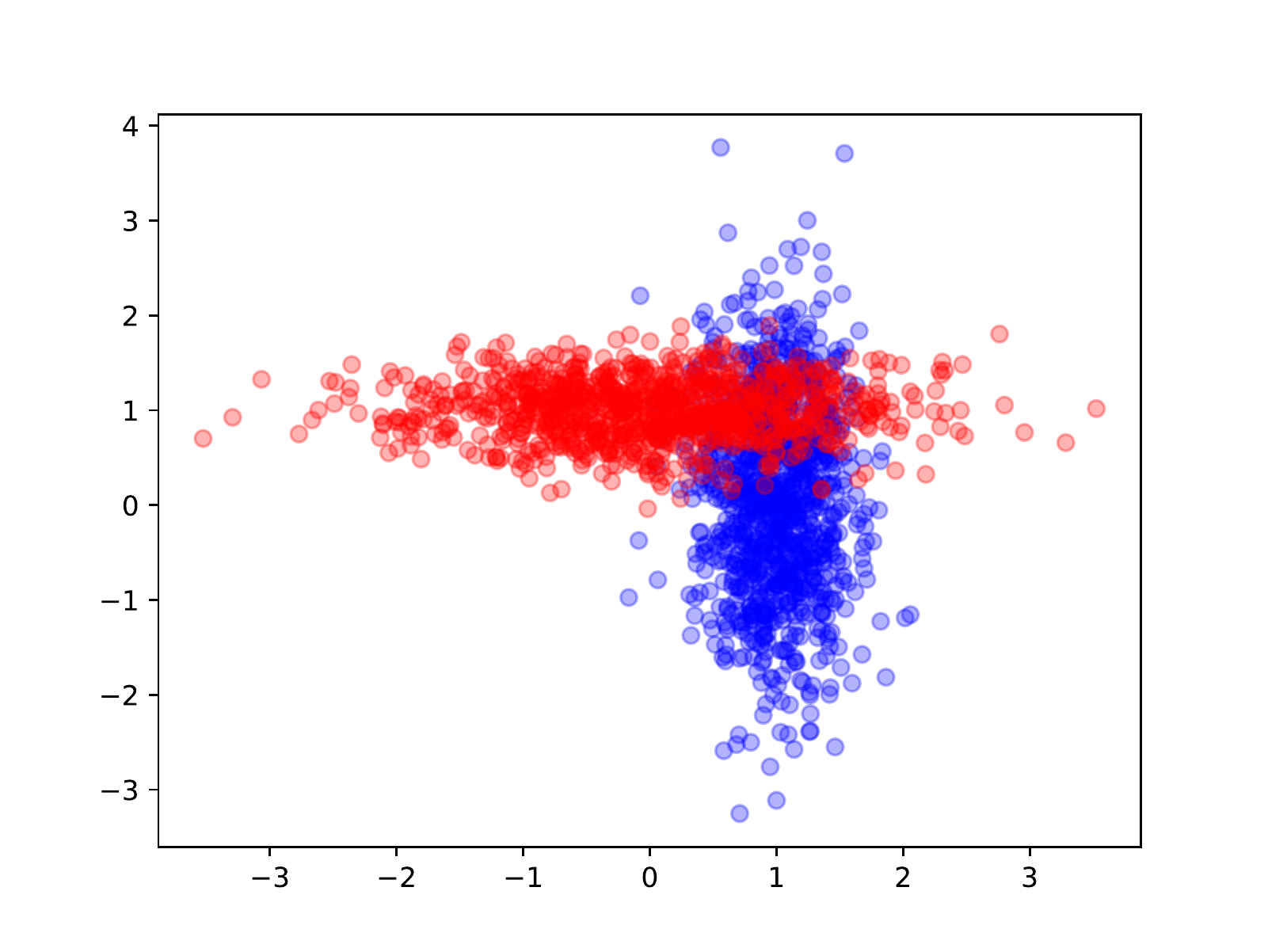}
    \caption{Original Space}
    \end{subfigure}
    \begin{subfigure}{0.48\textwidth}
    \centering
    \includegraphics[width=\textwidth]{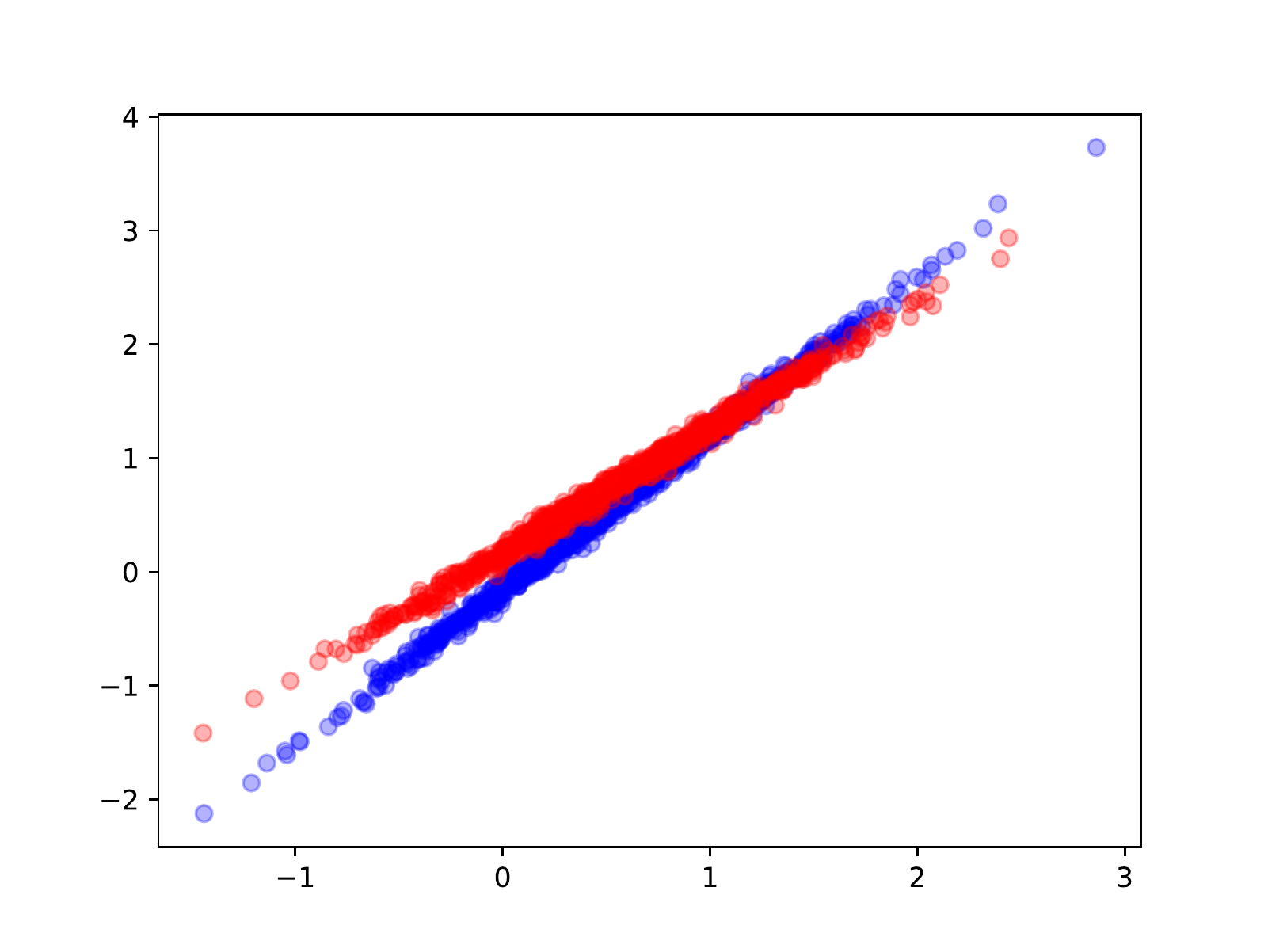}
    \caption{Latent Space}
    \end{subfigure}
    \caption{A verification of the theoretical results from Corollary \ref{corollary: two gaussians with xxt}. (a) Gaussian matrices that have covariances with nearly orthogonal top eigenvectors. (b) After stretching along the sum of the means, we align the major axes of the Gaussians.}
    \label{fig: Verification of Translation Alignment}
\end{figure}

\newpage 

\section{Proof of Theorem \ref{theorem: two layer linear autoencoder}: Gradient Flow Analysis for Linear Autoencoders}
\label{appendix: Gradient flow solution linear autoencoders}
\begin{proof}
Gradient descent using learning rate $\gamma$ proceeds as follows: 
\begin{align*}
    A^{(t+1)} &= A^{(t)} + \gamma (X - A^{(t)} B^{(t)} X) X^T {B^{(t)}}^T\\
    B^{(t+1)} &= B^{(t)} + \gamma {A^{(t)}}^T (X - A^{(t)} B^{(t)} X) X^T     
\end{align*}
We first show by induction that under the given initialization, $A^{(t)}, B^{(t)}$ are always polynomials in $XX^T$, and hence, only the the singular values of $A^{(t)}, B^{(t)}$ need be considered in the update.   The base case holds by definition, and so we assume that the same is true at time $t$.  If $A^{(t)}, B^{(t)}$ are polynomials in $XX^T$ at time $t$, then $A^{(t)} = U_X \Sigma_A^{(t)} U_X^T$ and $B^{(t)} = U_X \Sigma_B^{(t)} U_X^T$ where $U_X$ is the left singular vectors of $X$ and $\Sigma_A^{(t)}, \Sigma_B^{(t)}$ are diagonal matrices that are polynomials in $Sigma_X$.  Then at time $t+1$: 
\begin{align*}
    A^{(t+1)} &= U_X \Sigma_A^{(t)} U_X^T +  \gamma(X - U_X \Sigma_A^{(t)} \Sigma_B^{(t)} U_X^T X) X^T U_X \Sigma_B^{(t)} U_X^T \\
    &= U_X \left( \Sigma_A^{(t)} + \gamma(\Sigma_X - \Sigma_A^{(t)} \Sigma_B^{(t)} \Sigma_X ) \Sigma_X \Sigma_B^{(t)} \right) U_X^T.
\end{align*}
Thus, $A^{(t+1)}$ is of the form $U_X \Sigma_A^{(t+1)} U_X^T$ where $\Sigma_A^{(t+1)}$ is a polynomial in $\Sigma_X$. Hence, $A^{(t+1)}$ is a polynomial in $XX^T$.  We can proceed analogously for $B^{(t+1)}$ and induction is complete.  The result above implies that we need only consider gradient descent updates on $\Sigma_A^{(t)}, \Sigma_B^{(t)}$, which proceed as follows: 
\begin{align*}
    \Sigma_A^{(t+1)} &= \Sigma_A^{(t)} + \gamma (\Sigma_X - \Sigma_A^{(t)} \Sigma_B^{(t)} \Sigma_X ) \Sigma_X \Sigma_B^{(t)} \\
    \Sigma_B^{(t+1)} &= \Sigma_B^{(t)} + \gamma \Sigma_A^{(t)} (\Sigma_X - \Sigma_A^{(t)} \Sigma_B^{(t)} \Sigma_X ) \Sigma_X .
\end{align*}

By re-arranging terms and taking the limit as $\gamma \to 0$ (i.e. gradient flow), we get: 
\begin{align*}
    d\Sigma_A &= (\Sigma_X - \Sigma_A \Sigma_B \Sigma_X ) \Sigma_X \Sigma_B \\
    d\Sigma_B &= \Sigma_A (\Sigma_X - \Sigma_A \Sigma_B \Sigma_X ) \Sigma_X  ~ ;
\end{align*}
where $\Sigma_A, \Sigma_B$ are functions of time $t$ (i.e. $\Sigma_A = \Sigma_A(t)$).  

By the balancedness equation (i.e. multiplying $d\Sigma_A$ by $\Sigma_A$ and multiplying $d\Sigma_B$ by $\Sigma_B$), the above yields:  
\begin{align*}
    \Sigma_A d\Sigma_A  &= \Sigma_B d\Sigma_B.
\end{align*}

Now by integration coordinate-wise in time, we get: 
\begin{align}
\label{equation: diagonal balancedness equation}
    {\Sigma_A^{(t)}}^2 - {\Sigma_A^{(0)}}^2 = {\Sigma_B^{(t)}}^2 - {\Sigma_B^{(0)}}^2.
\end{align}

Lastly, we use the condition that at the end of training, our neural network must interpolate the data.  We denote $\Sigma_A = \Sigma_A^{(\infty)}$  and $\Sigma_B = \Sigma_B^{(\infty)}$.  
\begin{align}
\label{equation: interpolation at end of training}
    f(X) = X \implies \Sigma_A \Sigma_B \Sigma_X = \Sigma_X    \implies \Sigma_A \Sigma_B = I_r ~;
\end{align}
where $I_r$ is a diagonal matrix of $r$ 1's (with $r$ denoting the rank of $X$). Now for each diagonal entry of $\Sigma_A, \Sigma_B$ (denoted ${\Sigma_A}_{i,i}, {\Sigma_B}_{i,i}$), Equations \eqref{equation: diagonal balancedness equation} and \eqref{equation: interpolation at end of training} yield a system of 2 equations with 2 unknowns, which can be solved exactly.  For ease of notation let ${\Sigma_A}_{i,i} = a_i$, ${\Sigma_B}_{i,i} = b_i$, ${\Sigma_x}_{i,i} = x_i$. Summarizing below for a single diagonal entry, we have: 
\begin{align*}
    a_i^2 - {a_i^{(0)}}^2 &= b_i^2 - {b_i^{(0)}}^2 \\
    a_i b_i &= 1 \\
\end{align*}

Let $c^2 = b_0^2 - a_0^2$.  Reducing the system of equations above, we have: 
\begin{align*}
    &\frac{1}{b_i^2} = b_i^2 - c^2 \\
    & \implies b_i^4 - c^2 b_i^2 -1 = 0 \\
    &\implies \left(b_i^2 - \frac{c^2}{2}\right)^2 = 1 + \frac{c^4}{4} \\
    &\implies b_i = \sqrt{\sqrt{\left(1 + \frac{c^4}{4} \right)} + \frac{c^2}{2}}.
\end{align*}

We take the positive square root since this yields the solution closest to the initialization.  
\end{proof}
\newpage 

\section{Proof of Corollary \ref{corollary: two gaussians after autoencoding}: Alignment of 2D Gaussians after autoencoding}
\label{appendix: two gaussians after autoencoding}

\begin{proof}
Let $\lambda_i$ be the $i$th eigenvalue of $XX^T$. By Theorem \ref{theorem: two layer linear autoencoder}, if we initialize $A^{(0)}=0$ and  $B^{(0)} = (\frac{1}{n} XX^T) = U_X \Sigma_B U_X^T$, then after gradient descent we have:
$$B = U_X \Sigma_B U_X^T, (\Sigma_B)_i = \sqrt{\frac{\lambda_i^{2}}{2} + \sqrt{1 + \frac{\lambda_i^{4}}{4}}}.$$

For our particular setting, from Corollary \ref{appendis: two gaussians with xxt}, we know that $B^{(0)}$ converges as $n$ tends towards infinity:
$$B^{(0)} \xrightarrow{a.s} I + 2 J = (c + 3) I + 4\hat{v}\hat{v}^T$$
with eigenvectors $\hat{v} = \frac{[1,1]^T}{\sqrt{2}}$ and $\hat{u} = [1, -1]^T/\sqrt{2}$ with eigenvalues:
$\lambda_1 = (c + 7), \lambda_2 = (c+3)$. 

Moreover, we note that the mapping from $B^{(0)} \rightarrow B$ is continuous since it is a composition of applying the matrix square root and matrix polynomials, which are both continuous. Then using the continuous mapping theorem, we have that $B$ converges almost surely to $\bar{B}$, where $\bar{B}$ can be expressed as 
$$U_X = \begin{bmatrix} \hat{v} & \hat{u} \end{bmatrix} = \frac{1}{\sqrt{2}}\begin{bmatrix}1 & 1 \\ 1 & -1\end{bmatrix}, \quad \Sigma_{\bar{B}} =\begin{bmatrix} \alpha &  0 \\ 0 & \beta \end{bmatrix}, \quad \bar{B} = U_X \Sigma_{\bar{B}} U_X^T$$
for
$$\alpha = (\Sigma_{\bar{B}})_1 = \sqrt{\frac{(c+7)^{2}}{2} + \sqrt{1 + \frac{(c+7)^{4}}{4}}}, \beta = (\Sigma_{\bar{B}})_2 = \sqrt{\frac{(c+3)^{2}}{2} + \sqrt{1 + \frac{(c+3)^{4}}{4}}}.$$

We thus have that that:
$$\cos(Bq_{\Sigma_1}, Bq_{\Sigma_2})  \xrightarrow{a.s} \cos(\bar{B}q_{\Sigma_1}, \bar{B}q_{\Sigma_2})  = \frac{2 + \gamma \sqrt{c^2 + 4}}{2\gamma + \sqrt{c^2 + 4}} , \quad \gamma = \frac{(\alpha^2 - \beta^2)}{(\alpha^2 + \beta^2)}$$
for the above $\alpha$ and $\beta$. Moreover, as $k$ increases, the $\gamma$ goes up, and  $Bq_{\Sigma_1}$ and $Bq_{\Sigma_2}$ become more aligned. 
\end{proof}
\newpage 



\section{Proof of Theorem \ref{theorem:non_linear}: Stretching of ReLU Autoencoders}
\label{appendix: non_linear}
We begin by first proving the following lemma about properties of gradient descent in the nonlinear setting. 
\begin{lemma}
Let $f(x) = A\phi(Bx)$ be a one hidden layer ReLU neural network,  with $A \in \mathbb{R}^{d \times h}, B \in \mathbb{R}^{h \times d}$, and $\phi$ as a ReLU activation. Let $X = \{x, -x\}$ with $x \in \mathbb{R}^{d}$ denote our training set, and let $\eta = ||x||_2^2$. Then, the following hold:   

1. \textbf{Invariance:} If $A^{(0)} = xa^T$ and $B^{(0)} = bx^T$ for $a, b \in \mathbb{R}^{h}$, then  $A^{(t)} = x(a^{(t)})^T$, $B^{(t)} = b^{(t)}x^T$ for some $a^{(t)}, b^{(t)} \in \mathbb{R}^{h}$.

2. \textbf{Gradients:} If $A = xa^T$ and $B = bx^T$ for $a, b \in \mathbb{R}^{h}$, then the gradients of the autoencoder loss, $\mathcal{L}$ in \eqref{eq: Autoencoder Loss} with respect to $A$ and $B$ are:
\begin{align*}
      (\nabla_A \mathcal{L})_{ij} = \begin{cases}
        2 \eta (\eta a^T \phi(b) - 1) x_i b_j & b_j > 0\\
        -2 \eta (\eta a^T \phi(b) + 1) x_i b_j & b_j \leq 0\\
    \end{cases} ~~~~   (\nabla_B \mathcal{L})_{ij} = \begin{cases}
        2 \eta (\eta a^T \phi(b) - 1) x_j a_i & b_i > 0\\
        -2 \eta (\eta a^T \phi(b) + 1) x_j a_i & b_i \leq 0\\
    \end{cases}.
\end{align*}
\label{theorem:relu_lemma}
\end{lemma}
\begin{proof}
We first note that:
\begin{align*}
    \nabla_A L &= 2(A \phi(Bx) - x)\phi(Bx)^T + 2(A \phi(-Bx) + x)\phi(-Bx)^T\\
    &= 2\left(A\phi(Bx)\phi(Bx)^T + A \phi(-Bx)\phi(-Bx)^T - x(Bx)^T\right)\\
    \nabla_B L &= 2 \left(\left(A^T (A \phi(Bx) -x)\right) \odot \phi'(Bx) - \left(A^T (A \phi(-Bx) +x)\right) \odot \phi'(-Bx) \right)x^T.
\end{align*}

Moreover, if we can express $A=xa^T$ and $B=bx^T$, our gradients simplify to:
\begin{align*}
    \nabla_A L &= 2x \left(a^T \eta^2 \phi(b)\phi(b)^T + a^T \eta^2 \phi(-b)\phi(-b)^T - \eta b^T \right)\\
    &= 2\eta x  \left(\eta a^T (\phi(b)\phi(b)^T + \phi(-b)\phi(-b)^T) - b^T\right)\\
    \nabla_B L &= 2 \eta \left(\left(\eta a a^T \phi(b) -a \right) \odot \phi'(b) - \left(\eta aa^T \phi(-b) +a \right) \odot \phi'(-b) \right)x^T
\end{align*}

using the fact that $\phi(z) - \phi(-z) = z$. From these expressions, we find that we can express $\nabla_A L = x a'^T$ and $\nabla_B L = b' x^T$, proving our invariance property.

Now suppose that $A=xa^T$ and $B=bx^T$. Then we have:
$$(\nabla_A L)_{ij} = 2 \eta  (\eta a^T \phi(b) - 1) x_i b_j, \quad \text{if $b_j > 0$}$$
$$(\nabla_A L)_{ij} = 2 \eta  (-\eta a^T \phi(-b) - 1) x_i b_j, \quad \text{if $b_j \leq 0$}$$
and
$$(\nabla_B L)_{ij} = 2 \eta (\eta a^T \phi(b) - 1) x_j a_i \quad \text{if $b_i > 0$}$$
$$(\nabla_B L)_{ij} = -2 \eta (\eta a^T \phi(-b) + 1) x_j a_i  \quad \text{if $b_i \leq 0$}$$
which proves our second claim.

We empirically verify this result up to the third decimal place by training a network with SGD with learning rate $0.1$ with $x$ chosen as a uniformly random vector with dimension $1001$ and norm 2.
\end{proof}

The form of the gradients in Lemma \ref{theorem:relu_lemma} allows us apply a similar gradient flow analysis as in Theorem \ref{theorem: two layer linear autoencoder} to get a closed form for $A^{(\infty)}, B^{(\infty)}$. In the following lemma, we characterize the trajectories of $a$ and $b$ for $A=xa^T$ and $B=bx^T$ and establish equilibrium conditions on $a$ and $b$.
\begin{lemma}
\label{lemma: ReLU equilibrium}
In the setting of Lemma \ref{theorem:relu_lemma}, let $A^{(0)} = x(a^{(0)})^T$ and $B^{(t)} = b^{(0)}x^T$.  After $t$ steps of gradient descent on $A$ and $B$,  $A^{(t)} = x(a^{(t)})^T$ and $B^{(t)} = b^{(t)}x^T$ such that:
$$(a_i^{(t)})^2 - (a_i^{(0)})^2 = (b_i^{(t)})^2 - (b_i^{(0)})^2 .$$
Moreover, at equilibrium time step $t'$, the following hold:
$$(a^{(t')})^T\phi(b^{(t')}) = \frac{1}{\eta}, \quad (a^{(t')})^T\phi(-b^{(t')}) = \frac{-1}{\eta}.$$
\end{lemma}
\begin{proof}
Using our closed form expressions for the gradient in Theorem \ref{theorem:relu_lemma}, we notice that:
$$A_{ij} dA_{ij} = B_{ji} dB_{ji}.$$
We can integrate of both sides of the equation to get the values of $A$ and $B$ at time $t$:
$$(A_{ij}^{(t)})^2 - (A_{ij}^{(0)})^2 = (B_{ji}^{(t)})^2 - (B_{ji}^{(0)})^2.$$
By our invariance property, we know that $A^{(t)}$ can be expressed as $x(a^{(t)})^T$ and $B^{(t)}$ can be expressed as $b^{(t)}x^T$. We can thus rewrite this form as:
$$(a_i^{(t)})^2 - (a_i^{(0)})^2 = (b_i^{(t)})^2 - (b_i^{(0)})^2 .$$

Moreover, we note that at equilibrium time step $t'$ our gradients must be 0, which gives us 
$$(a^{(t')})^T\phi(b^{(t')}) = \frac{1}{\eta}, \quad (a^{(t')})^T\phi(-b^{(t')}) = \frac{-1}{\eta}.$$
\end{proof}

Lemma \ref{lemma: ReLU equilibrium} allows us to provide closed forms for $A^{(\infty)}, B^{(\infty)}$. We are now ready to prove the overarching Theorem \ref{theorem:non_linear}.
\begin{proof}

Define $w^T = \begin{bmatrix} x^T & -x^T \end{bmatrix}$. Suppose that we initialize $B = bx^T$ and $A = xa^T$ such that $b = \beta w$ and $a = \alpha w$. Then we note that $a^T\phi(b) = -a^T\phi(-b) = \alpha \beta \eta$. We can easily see from Lemma \ref{theorem:relu_lemma} that such a parameterization is invariant across gradient steps, since the gradients take the form:
$$(\nabla_A L)_{ij} = 2 \eta \beta (\eta^2 \alpha \beta  - 1)x_i x_j = \beta C_{\alpha, \beta} x_i x_j$$
and
$$(\nabla_B L)_{ij} = 2 \eta \alpha (\eta^2 \alpha \beta  - 1)x_j x_i = \alpha C_{\alpha, \beta} x_i x_j$$
for $C_{\alpha, \beta} = 2 \eta  (\eta^2 \alpha \beta  - 1)$.

Then from Lemma \ref{theorem:relu_lemma}, we know that for equilibrium time step $t'$, $a^{(t')}$ can be expressed as $\alpha^{(t')} w$ and $b^{(t')}$ can be expressed as $\beta^{(t')}w$. Then we can use our findings from point 3 of the above Lemma:
\begin{align*}
    (\alpha^{(t')})^2 -  (\alpha^{(0)})^2 =   (\beta^{(t')})^2 -  (\beta^{(0)})^2\\
    \alpha^{(t')} \beta^{(t')} = \frac{1}{\eta^2}.
\end{align*}
Solving this system of equations,  for $\alpha^{(0)}=0$ and $\beta^{(0)}=1$ gives us
\begin{align*}
    ((\beta^{(t')})^2-1)(\beta^{(t')})^2 &= \frac{1}{\eta^4}\\
    (\beta^{(t')})^2 &= \frac{1}{2} \left(\frac{\sqrt{\eta^4 + 4}}{\eta^2} +1\right).
\end{align*}

Then we have that
$$||Bx||_2^2 = 2\eta^3(\beta^{(t')})^2 = \eta^3 \left(\frac{\sqrt{\eta^4 + 4}}{\eta^2} +1\right).$$
\end{proof}

\newpage

\section{Experimental Setup}
\label{appendix: experimental_setup}
We now describe the training setup and hardware for the experiments conducted in this work.  For all linear autoencoders, we use the Theorem \ref{theorem: two layer linear autoencoder} to directly compute the solution given by gradient flow.

\paragraph*{Alignment of Drug Signatures} To train the 1 hidden layer, leaky ReLU autoencoder on data from CMap, we use the available code from \cite{COVIDAutoencoding}.  We follow the training procedure from \cite{COVIDAutoencoding} for training the non-linear autoencoder.  Namely, we use Adam with a learning rate of $10^{-4}$, and we use an architecture with 1024 hidden units.  For the linear autoencoders, we use 911 hidden units.  

\paragraph*{GloVe} To train the leaky ReLU autoencoder on the GloVe data, we use a 1 hidden layer autoencoder with hidden dimension 400. We use SGD with a learning rate of 0.1 and momentum of 0.9. The weights of the linear autoencoders were computed using the closed form in Theorem \ref{theorem: two layer linear autoencoder}.

\paragraph*{Hardware} For the experiments on CMap, we use a server with 128GB RAM, 16 cores, and 2 NVIDIA 2080 Ti GPUs (12GB).  For the GloVe experiments, we used an NVIDIA v100 with 32GB of memory.

\newpage

\section{GloVe Appendix}
\label{appendix: glove}
\subsection{Interpretation of the Top Frequency Directions of GloVe}

Let $X \in \mathbb{R}^{300 \times 500,000}$ be the GloVe embeddings of the 500,000 most frequent words. We compute the top singular vector $u_1 \in \mathbb{R}^{300}$ of $X$. In Figure \ref{fig:top_glove_dir}, for each word with embedding $w$, we plot $\cos(w, u_1)$ against the frequency ranking for that word. We find that frequent words are more aligned with the top singular vector.

\label{appendix:glove_frequency}
\begin{figure}[!htbp]
    \centering
    \includegraphics[height=2.2in]{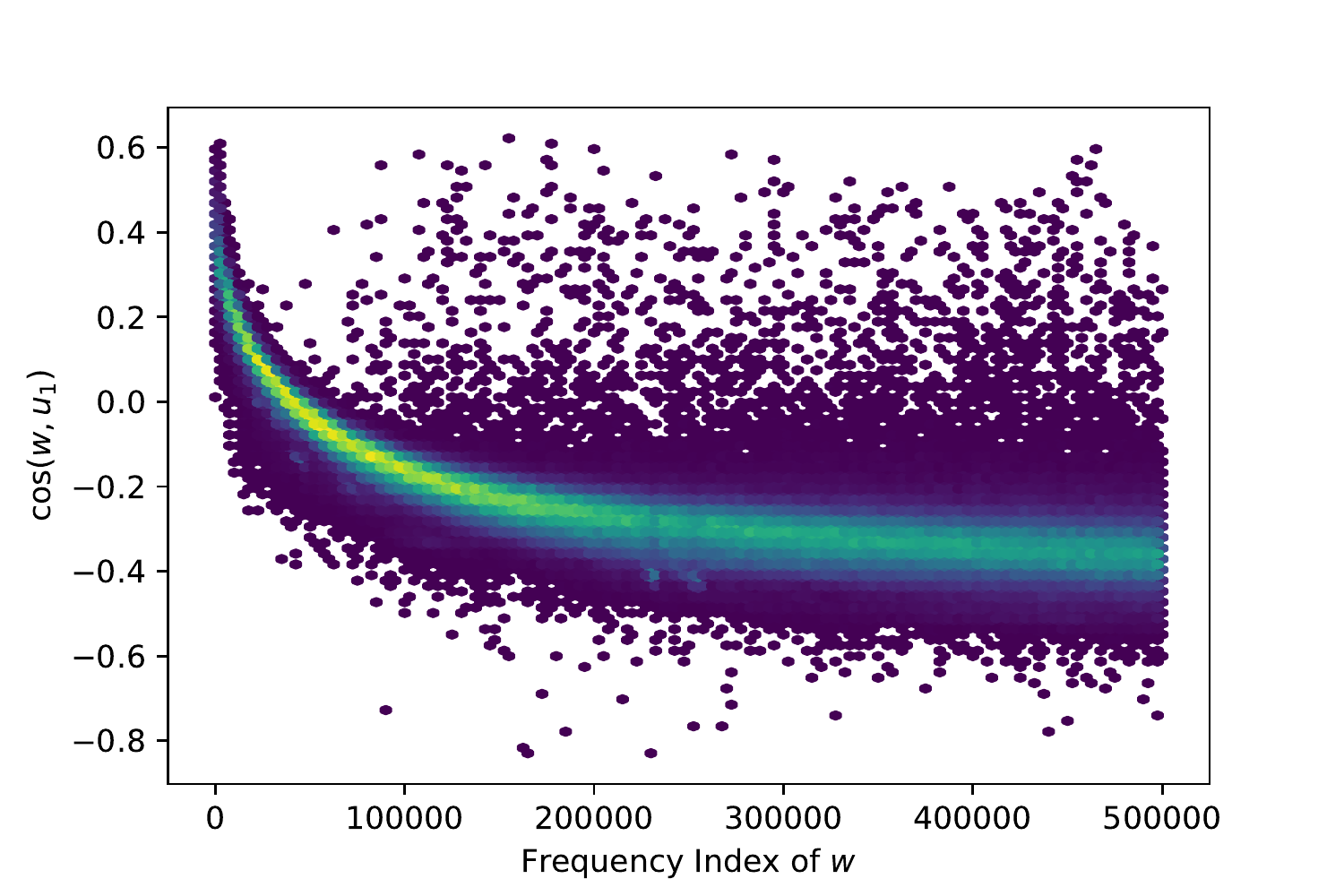}
    \caption{We plot the alignment of word embeddings with the top singular vector of the GloVE dataset against the word's frequency ranking. Here, green indicates higher density.}
    \label{fig:top_glove_dir}
\end{figure}

\subsection{Alignment of meaningful directions before and after autoencoding}
Let $(a_1, a_2)$ and $(b_1, b_2)$ be two pairs of word embeddings such that the corresponding words have the same relationship. An example might be the embeddings of (man, woman) and (king, queen). Then for encoding matrix $B$, we plot the cosine similarity of the translation vector before encoding (i.e $\cos(a_2 - a_1, b_2-b_1)$) and after encoding (i.e $\cos(Ba_2 - Ba_1, Bb_2-Bb_1)$).

We take every combination of pairs with the same relationship from the Google Analogy test set (this results in 19430 combinations), and plot using the linear autoencoder with asymmetric spectral initialization ($k=2$), as well as a LeakyReLU autoencoder (Figure \ref{fig:glove_corr_lots}). We find that both autoencoders increase the alignment of the word pairs, but using the linear autoencoder results in more alignment. For a linear autoencoder with $k=2$, 76\% of the autoencoded pairs have cosine similarity above 0.5.

\label{appendix: glove_alignment}
\begin{figure}[!htbp]
    \centering
    \begin{subfigure}{0.47\textwidth}
    \includegraphics[width=\textwidth]{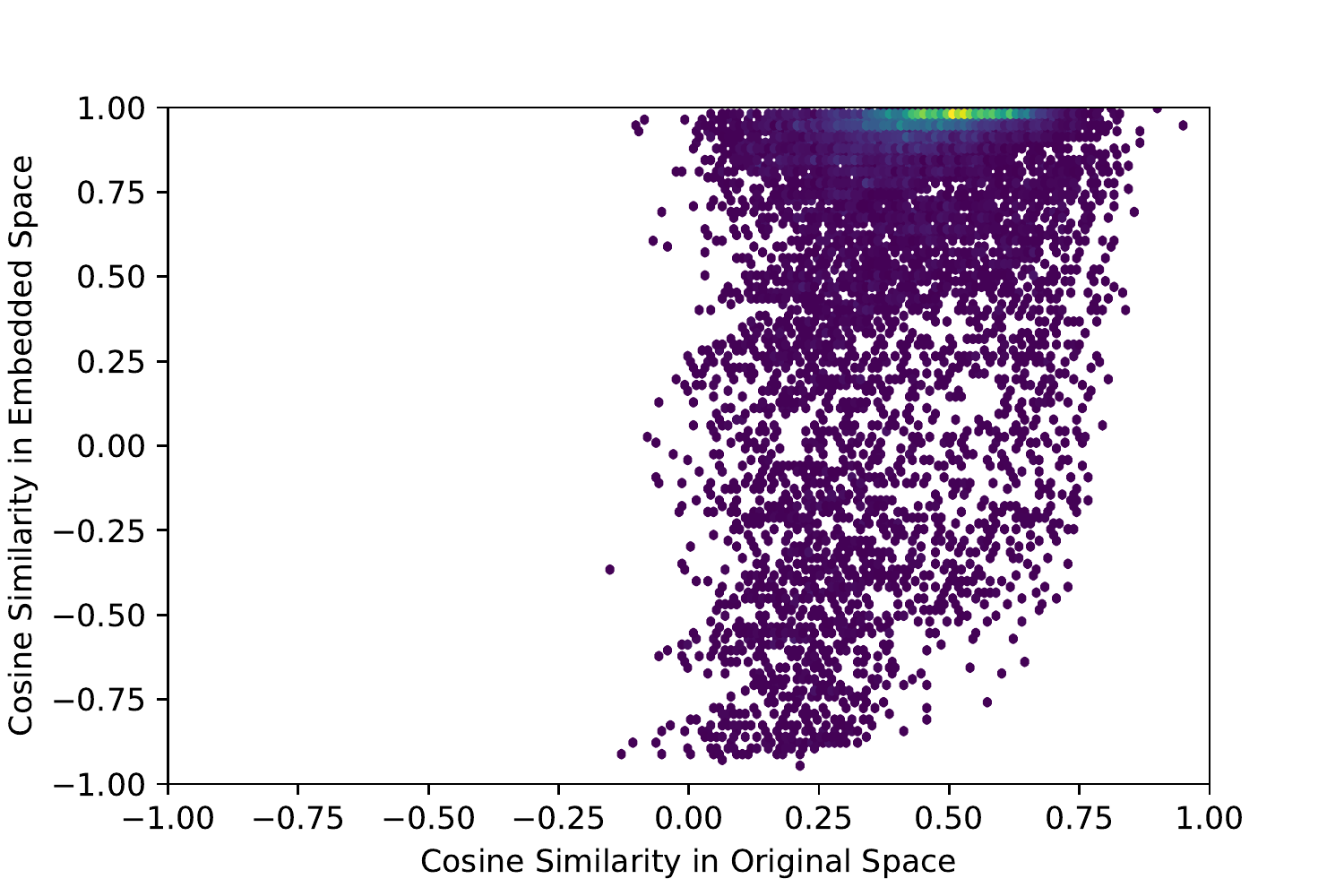}
    \caption{Linear Autoencoder (k=2)}
    \end{subfigure}
    \begin{subfigure}{0.47\textwidth}
    \includegraphics[width=\textwidth]{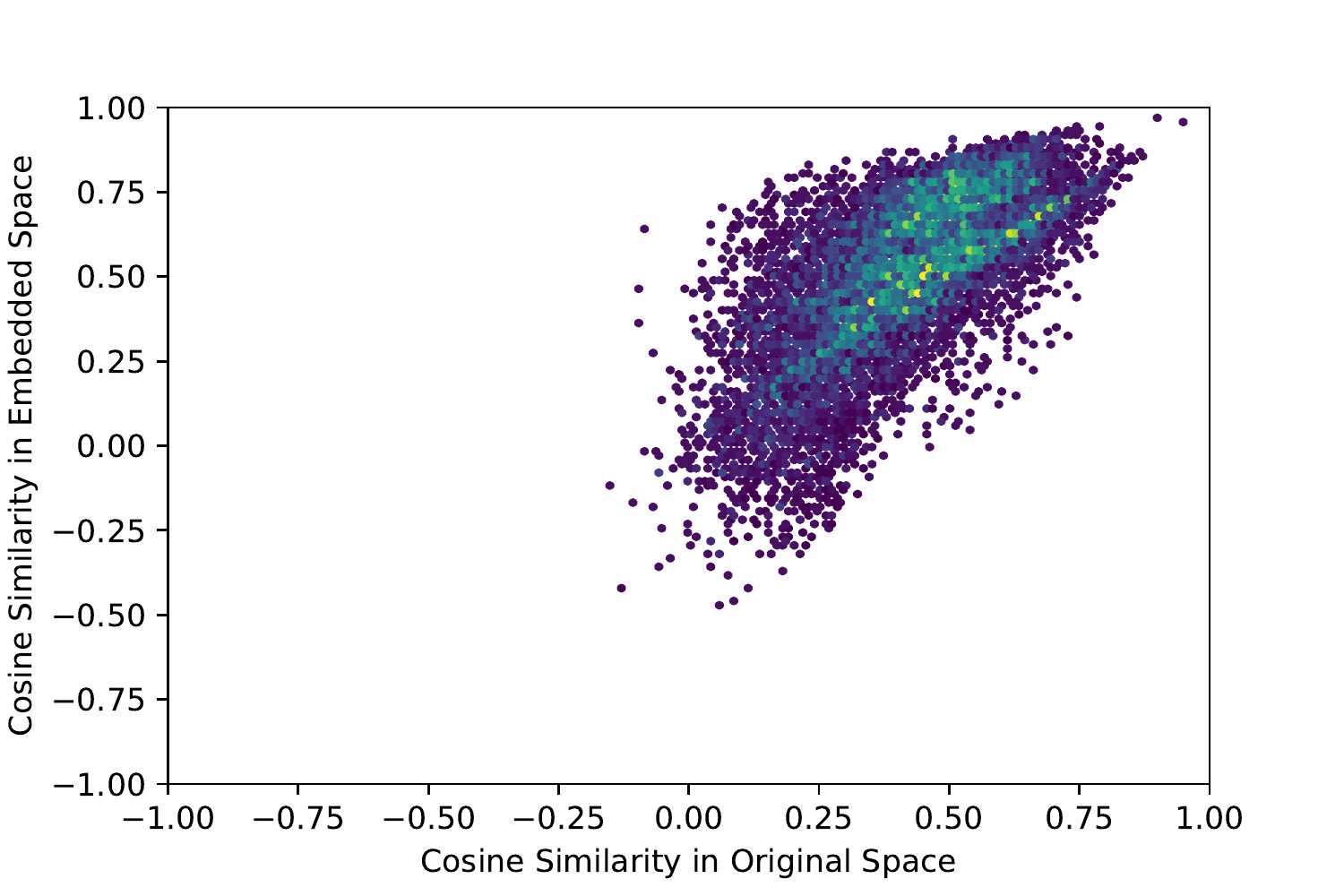}
    \caption{Leaky ReLU Autoencoder}
    \end{subfigure}
    \caption{We plot the cosine similarity of word pairs with the same relationship before and after autoencoding.  Here, green indicates higher density.}
    \label{fig:glove_corr_lots}
\end{figure}

\subsection{Visualizing alignment before and after autoencoding using MDS}
We visualize the latent space before and after autoencoding using multidimensional scaling (MDS). For a given relationship from the Google Analogy test set, we take the set of all words that appear in a pair for that relationship. We perform MDS with two components on the word embeddings of that set before and after autoencoding. Using the coordinates derived from MDS, we plot the pairs to visualize alignment.

In Figures \ref{fig:dense_negation} and \ref{fig:dense_superlative}, we visualize the alignment of two relationships: negation (i.e aware $\rightarrow$ unaware) and superlative (i.e cold $\rightarrow$ coldest). We find that embeddings derived from a linear autoencoder with asymmetric spectral initialization align the translation vectors for word pairs with the same relationship.
\label{appendix:glove_mds}

\begin{figure}[!htbp]
    \centering
    \begin{subfigure}{0.45\textwidth}
     \includegraphics[width=\textwidth]{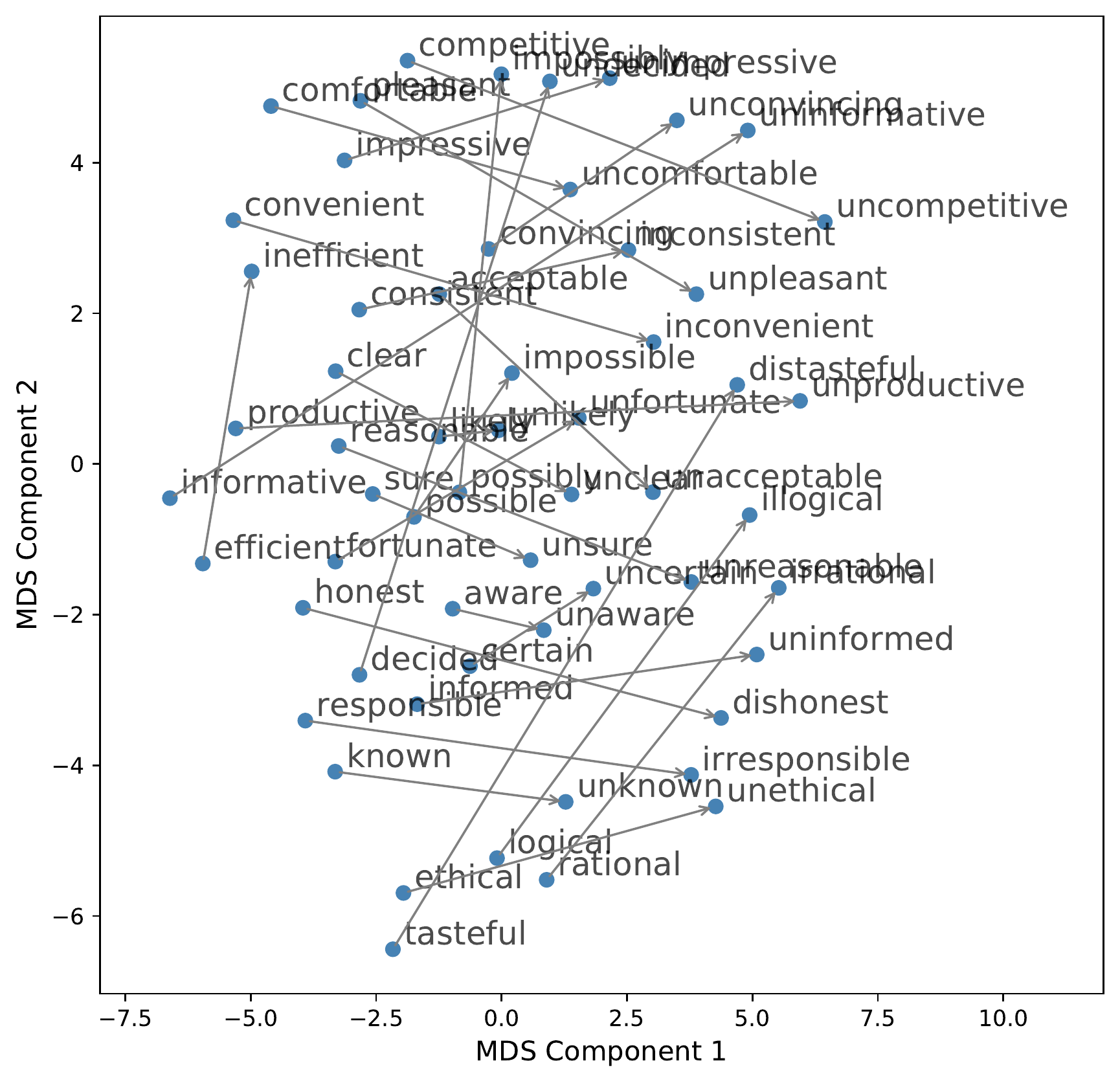}
     \caption{Original Space}
    \end{subfigure}
    \begin{subfigure}{0.45\textwidth}
     \includegraphics[width=\textwidth]{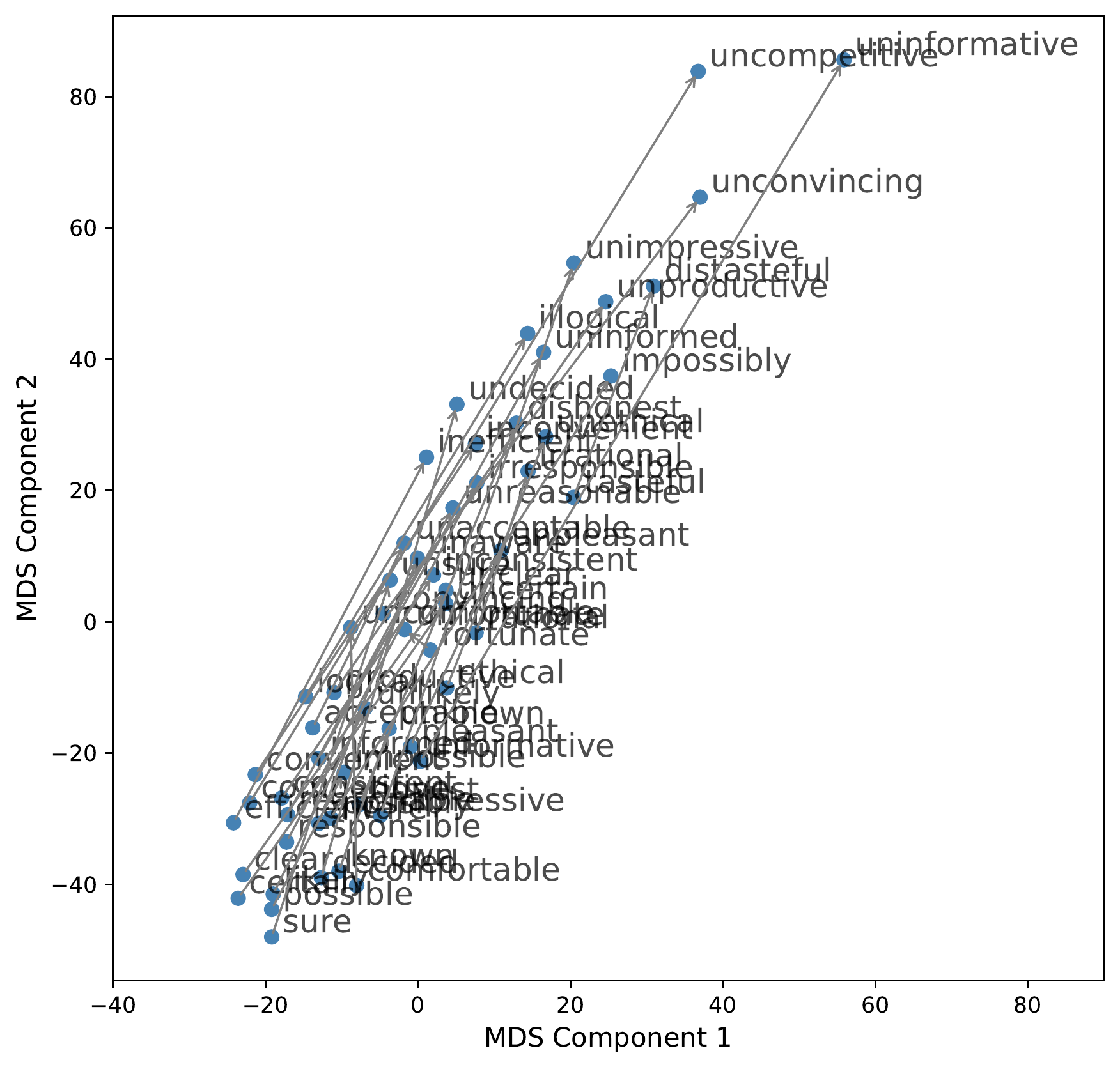}
     \caption{Linear Autoencoder (k=2)}
    \end{subfigure}
    \caption{We visualize the embeddings of words with a negation relationship using MDS both in the original space and linearly autoencoded space (k=2).}
    \label{fig:dense_negation}
\end{figure}

\begin{figure}[!htbp]
    \centering
    \begin{subfigure}{0.45\textwidth}
     \includegraphics[width=\textwidth]{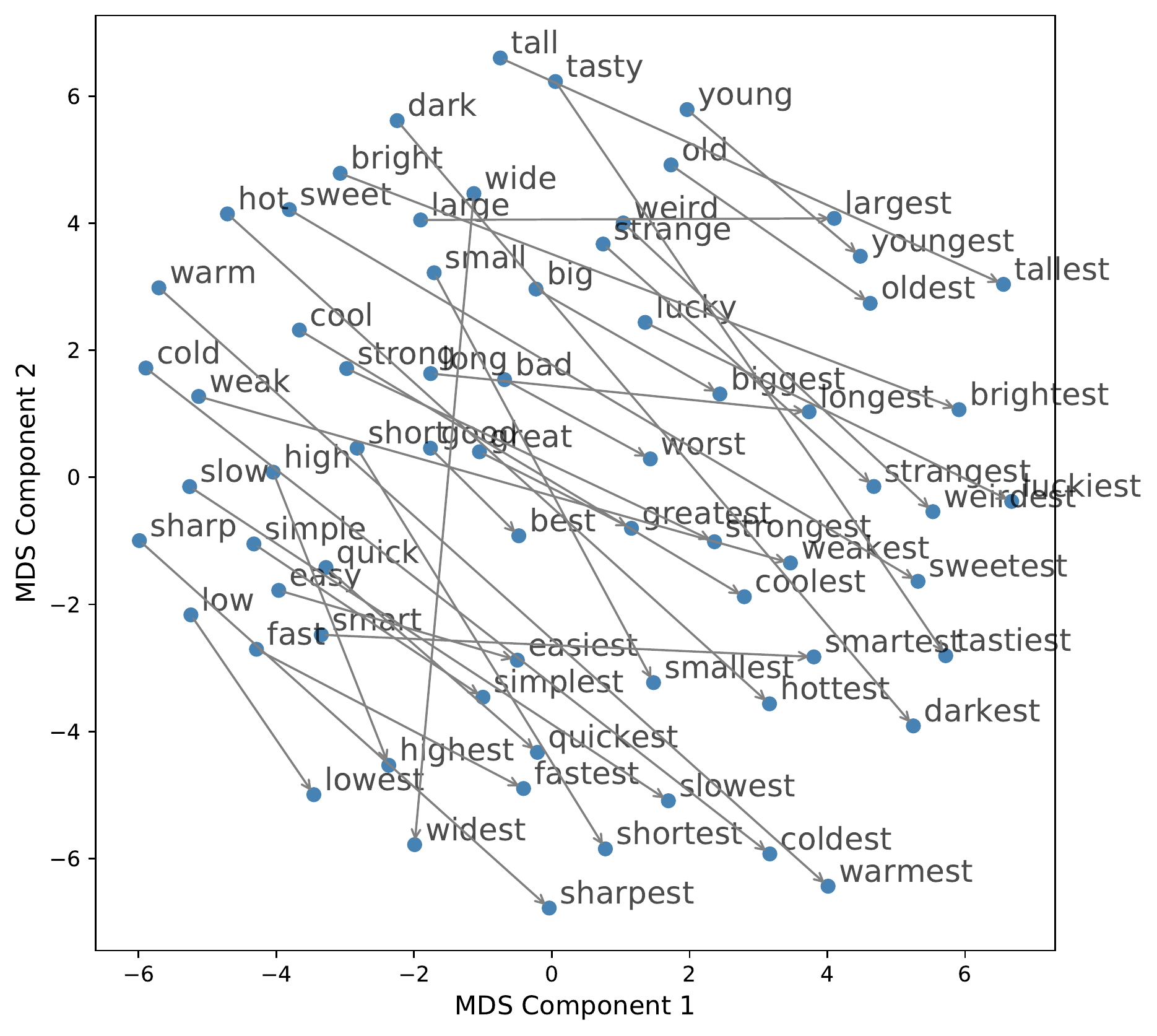}
     \caption{Original Space}
    \end{subfigure}
    \begin{subfigure}{0.45\textwidth}
     \includegraphics[width=\textwidth]{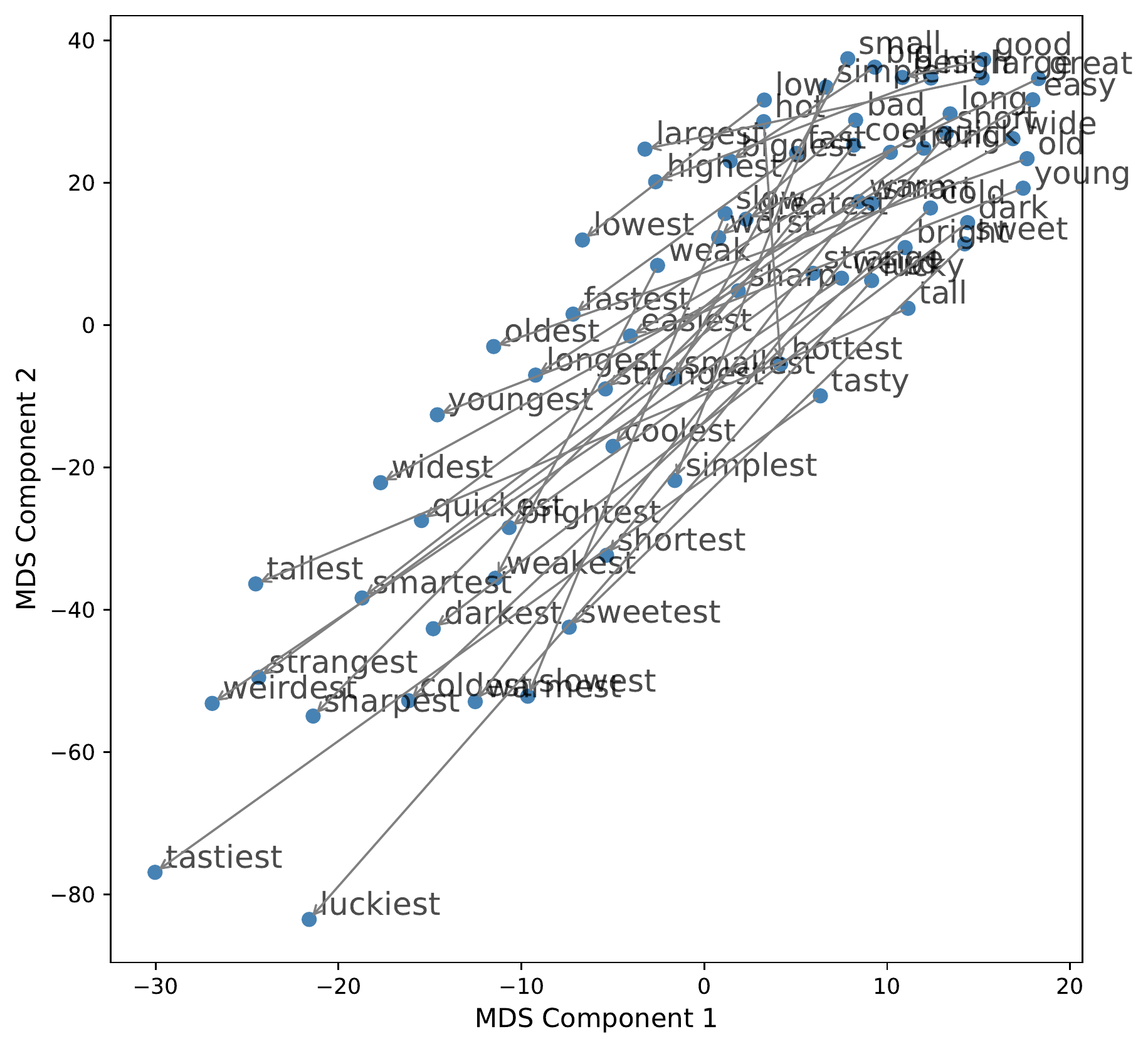}
     \caption{Linear Autoencoder (k=2)}
    \end{subfigure}
    \caption{We visualize the embeddings of words with a superlative relationship using MDS both in the original space and linearly autoencoded space (k=2)}
    \label{fig:dense_superlative}
\end{figure}

\subsection{Reference Point Correction}
\label{appendix: reference_point_correction}
Word similarity metrics capture the \textit{locality} of a latent space by computing whether similar words have similar embeddings \cite{chung2017learning}. To capture the similarity of two embeddings $a$ and $b$, these metrics typically compute their cosine similarity against a fixed reference point at the origin $o = \mathbf{0}$. 
Specifically, the similarity of the embeddings in the original space would be $\cos(a-o, b-o) = \cos(a,b)$ and the similarity after encoding with a matrix $B$ would be $\cos(B(a-o), B(b-o))=\cos(Ba, Bb)$.

 However, from our theory, we know that autoencoding will stretch along the top singular vectors. The point $o = \mathbf{0}$ is thus a poor reference point since it will not stretch along the rest of the dataset. We instead pick an reference point $o'$ that will stretch alongside the top principal component: for the original GloVe dataset $X \in \mathbb{R}^{d \times N}$ with top left singular value and vector $(\lambda_1, u_1)$, we set our reference point as $o' = \lambda_1 u_1$. This is equivalent to simply \textit{translating} our latent space before computing word similarity to a new origin point (note that translation will not change alignment).
 
 In Figure \ref{app_fig:ref_point_corr} we show the effect of setting the reference point to $\lambda_1 u_1$. When choosing the default origin at $\mathbf{0}$, the angle between embeddings $w_1$ and $w_2$ becomes much larger as we stretch in the red direction. However, by choosing the corrected $o'$, the reference point stretches alongside the rest of the dataset, so the angle between $w_1 -o'$ and $w_2 - o'$ does not significantly change after autoencoding.

\begin{figure}[!htbp]
    \centering
    \includegraphics[width=\textwidth]{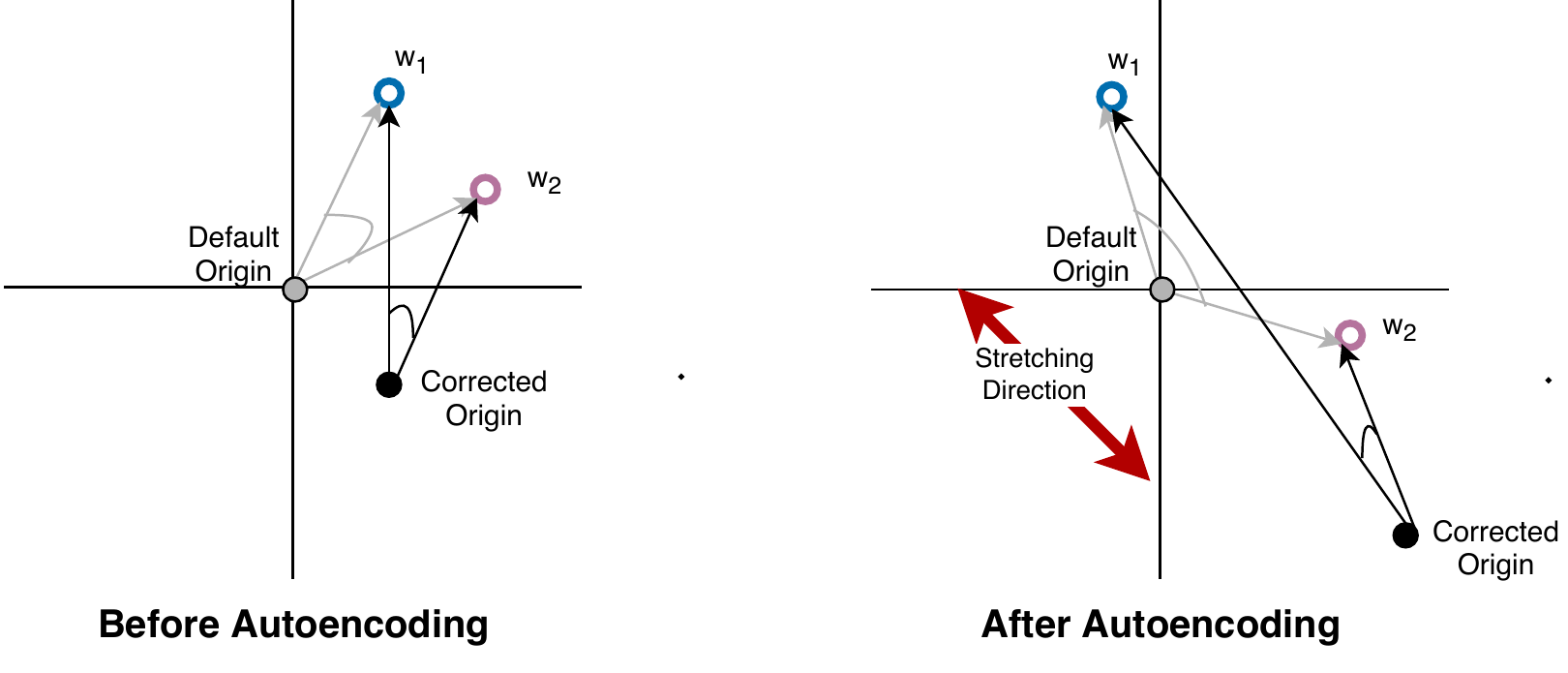}
    \caption{A schematic depicting reference point correction. When using a default reference point of $o = \mathbf{0}$ (the grey dot) the cosine similarity between embeddings $w_1$ and $w_2$ become distorted after stretching. However, if we choose a reference point along the stretching direction (the black dot), then the cosine similarity remains stable even after autoencoding.}
    \label{app_fig:ref_point_corr}
\end{figure}

In Table \ref{app_fig:ref_point_table}, we display the word similarity metric $\texttt{wsim}$ with and without reference point correction. As expected, we find that stretching using autoencoders degrades the word similarity metric, if we pick the default reference point $o=\mathbf{0}$. However, if we choose the corrected reference point $o'=\lambda_1 u_1$, the word similarity metrics remain largely stable, especially for the linear autoencoders.

\begin{table}[!htbp]
    \centering
    \caption{GloVE \texttt{wsim} with and without reference point correction.}
    \begin{tabular}{l|cc|cc}
    \toprule
         & \multicolumn{2}{c}{Default Origin} & \multicolumn{2}{c}{Corrected Origin}\\
        Model & MTurk-771 & MEN & MTurk-771 & MEN\\\midrule
        Original         & 0.715 & 0.804 & 0.629 &  0.739\\
        Linear AE (k=1)  & 0.431 & 0.457 & 0.628 &  0.738 \\
        Linear AE (k=2)  & 0.253 & 0.221 & 0.611 & 0.729\\
        Leaky ReLU AE    & 0.559 & 0.570 & 0.592 &  0.598 \\
        
        \bottomrule
    \end{tabular}
    \label{app_fig:ref_point_table}
\end{table}



\end{document}